\documentclass{article}

\usepackage{iclr2027_conference,times}
\usepackage[utf8]{inputenc}
\usepackage[T1]{fontenc}
\usepackage{amsmath}
\usepackage{amssymb}
\usepackage{amsthm}
\usepackage{booktabs}
\usepackage{multirow}
\usepackage{algorithm}
\usepackage{algpseudocode}
\usepackage{graphicx}
\usepackage{capt-of}
\usepackage{microtype}
\usepackage{hyperref}
\hypersetup{hidelinks}
\usepackage{url}
\usepackage{tikz}
\usetikzlibrary{arrows.meta,positioning,decorations.pathreplacing}

\newtheorem{proposition}{Proposition}
\newif\ifarxiv
\arxivtrue
\iclrfinalcopy

\title{Constraints Are Graphs, Not Chains: Exact\\
Decoding for Diffusion Language Models}
\author{Jianchang Su\\
University of Connecticut
\And
Wei Zhang\\
University of Connecticut}
\date{}

\begin{document}
\raggedbottom
\maketitle
\lhead{Preprint}

\begin{abstract}
Diffusion language models (dLLMs) predict masked positions in arbitrary order, but their exact
constrained decoders still encode constraints as sequential languages, whose state must track
every unresolved dependency between positions.  For relational constraints this encoding grows
exponentially: for same-order copy, every finite automaton needs $4^k$ states, deterministic or
nondeterministic, and every context-free grammar has size $2^{\Omega(k)}$, while the factor graph
of the same relation has size $O(k)$ and a 16-entry peak table.
We introduce \emph{FactorDLM}, a training-free decoder that represents finite-domain relations as
a factor graph and, at each denoising step, conditions the model's mean-field prediction on that
graph exactly by variable elimination.  Decoding cost then grows exponentially with the induced
width of the constraint graph, which replaces automaton size as the governing parameter.  Because
a finite automaton is a chain-shaped factor graph, one compiler enforces syntax and nonlocal
relations together: on JSON records with cross-field references, a schema automaton alone leaves
references dangling, relational factors alone produce malformed JSON, and the combined plan is
valid on both counts, including on records of variable length.  Across nine relational benchmarks
and three backbones, every output satisfies every declared constraint at 0.4--6.9\% projection
overhead, where unconstrained decoding is 0--79\% valid, and compiled projection answers repeated
queries $13.6\times$ faster than CP-SAT with eight parallel workers.  Because model-free rules
solve three of five standard benchmarks, we construct benchmarks with exact chance and
fixed-template floors, on which selecting among exact constrained samples beats greedy
projection.  Which encoding is cheaper, sequential state or direct factors, depends on the
constraint and is computable before decoding begins.
\end{abstract}

\section{Introduction}

Masked diffusion language models (dLLMs) generate text by repeatedly predicting all masked
positions and committing a subset of them \citep{nie2025llada,ye2025dream}.  This arbitrary-order
mechanism suits infilling and parallel refinement, but it breaks the usual left-to-right recipe
for constrained decoding.  When each position is filtered independently, tokens that are
plausible in isolation can form an assignment that violates a global constraint.

Recent dLLM decoders enforce regular expressions, context-free grammars, lookahead verification,
and finite automata \citep{suresh2025dingo,mundler2025cfg,zhang2026lave,dang2026automata}, which
suit sequential syntax.  Many practical constraints, however, relate distant positions: equality
between fields, foreign-key consistency, graph coloring, and scheduling conflicts.  A
left-to-right state must then remember every earlier decision a later position depends on, so
the number of states can grow exponentially even when the relation graph has constant treewidth.

\begin{figure}[t]
\centering
\resizebox{0.8\textwidth}{!}{
\providecommand{\tokcell}[3]{%
  \draw[thick] (#1, #2) rectangle ++(0.55, 0.55);
  \node at ({#1+0.275}, {#2+0.275}) {\small #3};}
\providecommand{\maskcell}[2]{%
  \fill[black!12] (#1, #2) rectangle ++(0.55, 0.55);
  \draw[thick] (#1, #2) rectangle ++(0.55, 0.55);}
\begin{tikzpicture}[
    every node/.style={font=\normalsize, inner sep=2.5pt},
    arr/.style={-{Stealth[length=2.5mm]}, thick},
    fac/.style={rectangle, draw, thick, fill=black!85, minimum size=2.6mm, inner sep=0pt},
]

\begin{scope}[xshift=0cm]
\node (z1) at (0,   2.5) {$z_1$};
\node (z2) at (1.0, 2.5) {$z_2$};
\node (zd) at (2.0, 2.5) {$\cdots$};
\node (zL) at (3.0, 2.5) {$z_{2k}$};
\node[draw, fill=black!12, thick, rounded corners=1pt, minimum size=6mm,
      inner sep=1pt] (zA) at (4.0, 2.5) {\scriptsize $z_{2k\!+\!1}$};
\node (x1) at (0,   1.6) {$x_1$};
\node (x2) at (1.0, 1.6) {$x_2$};
\node (xL) at (3.0, 1.6) {$x_{2k}$};
\draw[arr] (z1) -- (z2);
\draw[arr] (z2) -- (zd);
\draw[arr] (zd) -- (zL);
\draw[arr] (zL) -- (zA);
\draw[arr] (z1) -- (x1);
\draw[arr] (z2) -- (x2);
\draw[arr] (zL) -- (xL);
\node at (2.0, 0.75) {$|\mathcal S|\ \geq\ 4^{k}$ after $k$ symbols};
\node at (2.0, -0.55) {(a) Constraint as sequential state};
\end{scope}

\begin{scope}[xshift=6.1cm]
\node (a1) at (0,   2.7) {$x_1$};
\node (a2) at (1.0, 2.7) {$x_2$};
\node (a3) at (2.0, 2.7) {$x_3$};
\node (b1) at (0,   1.3) {$x_4$};
\node (b2) at (1.0, 1.3) {$x_5$};
\node (b3) at (2.0, 1.3) {$x_6$};
\node[fac] (g1) at (0,   2.0) {};
\node[fac] (g2) at (1.0, 2.0) {};
\node[fac] (g3) at (2.0, 2.0) {};
\draw[thick] (a1) -- (g1) -- (b1);
\draw[thick] (a2) -- (g2) -- (b2);
\draw[thick] (a3) -- (g3) -- (b3);
\node at (2.75, 2.0) {$f_{=}$};
\node at (1.0, 0.75) {width $1$;\; $4^{2}$-entry tables};
\node at (1.0, -0.55) {(b) Constraint as factor graph};
\end{scope}

\begin{scope}[xshift=10.6cm]
\node[anchor=east] at (-0.15, 3.475) {$x^{t}$};
\tokcell{0}{3.2}{2}\maskcell{0.55}{3.2}\maskcell{1.10}{3.2}
\tokcell{1.65}{3.2}{2}\maskcell{2.20}{3.2}\tokcell{2.75}{3.2}{1}
\draw[arr] (1.65, 3.1) -- (1.65, 2.45);
\node[anchor=west, align=left, font=\small] at (1.9, 2.775)
  {dLLM unaries $u_i$; exact elimination\\ over $\sum_i u_i(x_i) + \sum_a f_a(x_{S_a})$};
\node[anchor=east] at (-0.15, 2.075) {$\hat x$};
\tokcell{0}{1.8}{2}\tokcell{0.55}{1.8}{4}\tokcell{1.10}{1.8}{1}
\tokcell{1.65}{1.8}{2}\tokcell{2.20}{1.8}{4}\tokcell{2.75}{1.8}{1}
\draw[arr] (1.65, 1.7) -- (1.65, 1.05);
\node[anchor=west, align=left, font=\small] at (1.9, 1.375)
  {commit the $\lceil |M_t|/t\rceil$ most-confident\\ slots; re-mask the rest};
\node[anchor=east] at (-0.15, 0.675) {$x^{t-1}$};
\tokcell{0}{0.4}{2}\maskcell{0.55}{0.4}\maskcell{1.10}{0.4}
\tokcell{1.65}{0.4}{2}\tokcell{2.20}{0.4}{4}\tokcell{2.75}{0.4}{1}
\draw[very thick] (2.20, 0.4) rectangle ++(0.55, 0.55);
\node at (2.9, -0.55) {(c) One denoising step};
\end{scope}

\begin{scope}[xshift=18.6cm]
\node[font=\large, anchor=north west, align=left] at (-0.1, 3.6)
  {JSON with cross-field\\ references};
\node[draw, rounded corners, fill=black!5, inner sep=4pt, anchor=north west,
      font=\large\ttfamily] at (-0.1, 2.55) {users: p, q};
\node[draw, rounded corners, fill=red!8, inner sep=4pt, anchor=north west,
      font=\large\ttfamily] at (-0.1, 1.85) {"on": "w"};
\node[font=\large, anchor=west] at (2.35, 1.55) {\textcolor{red!70!black}{undeclared}};
\node[draw, rounded corners, fill=green!10, inner sep=4pt, anchor=north west,
      font=\large\ttfamily] at (-0.1, 1.05) {"on": "q"};
\node[font=\large, anchor=west] at (2.35, 0.75) {exact};
\node[font=\large, anchor=north west] at (-0.1, 0.35)
  {automaton alone: 48--53\% dangling};
\node at (2.2, -0.55) {(d) Syntax and relations in one plan};
\end{scope}

\end{tikzpicture}}
\caption{\textbf{Constraints are graphs, not chains.}  (a) Prior exact decoders run inference in
the chain graphical model of an automaton \citep{dang2026automata}, which needs $4^k$ states for
the copy relation (Proposition~\ref{prop:copy}); (b) its factor graph has width one and 16-entry
tables; (c) one denoising step commits the most confident slots of an exactly valid assignment;
(d) a schema automaton is a chain-shaped factor graph, so schema and cross-field references
compile into one plan (\S\ref{sec:support}).}
\label{fig:overview}
\end{figure}
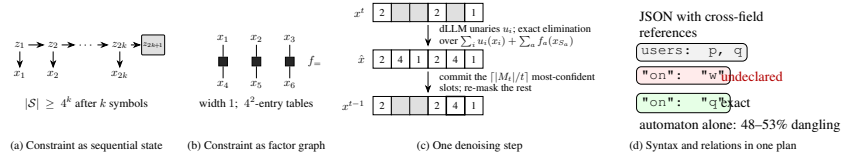

We therefore ask a representation question: \emph{why should an arbitrary-order generator encode
its constraints as left-to-right state?}  FactorDLM keeps fixed-slot, finite-domain constraints
as a factor graph: at each denoising step, dLLM logits become unary potentials, hard relations
become log factors with values zero or $-\infty$, and exact variable elimination returns a
proposal that satisfies every constraint.  Every proposal is jointly feasible, so each commitment
keeps a valid completion available and the schedule can commit the most confident positions
first; when induced width makes exact inference too expensive, the compiler refuses before
decoding.

Our contributions are as follows.
\begin{itemize}
  \item \textbf{Representation and theory.}  We hold constraints as factor graphs, of which
  sequential automata are the chain case (Figure~\ref{fig:overview}a), so decoding cost is
  exponential in induced width.  For same-order copy the separation is exponential: every
  automaton, deterministic or nondeterministic, needs $4^k$ states and every grammar has size
  $2^{\Omega(k)}$~\citep{filmus2011cfg}, while the direct factor encoding has width one and size
  $O(k)$; for counting and all-different the sequential encoding is the small one, so the cheap
  encoding is a per-constraint choice priced before allocation (Proposition~\ref{prop:copy}).
  Two further results extend exact decoding to the full vocabulary and show what a joint MAP over
  a padded grid adds to the local termination decision (Propositions~\ref{prop:quotient}
  and~\ref{prop:termination}).
  \item \textbf{One exact compiler for syntax and relations.}  A finite automaton is a
  chain-shaped factor graph, so one compiler composes a schema automaton with relational factors;
  on JSON with cross-field references each part alone fails on one axis and the combined plan is
  valid on both, including records of variable length and open-vocabulary fields
  (\S\ref{sec:support}).
  \item \textbf{Guarantee, cost, and boundary.}  Across nine relational benchmarks and three
  backbones, every output satisfies every declared constraint, where independent decoding is
  0--79\% valid (Table~\ref{tab:main}); projection adds 0.41--6.86\% latency, repeated queries run
  $13.6\times$ faster than CP-SAT with eight workers, and we report where our system stops.
  \item \textbf{Measurement.}  With validity guaranteed, we test what each benchmark measures:
  audits trace four published semantic results to a unique completion, a canonical target, or a
  degenerate baseline, and on benchmarks with exact floors, scoring rendered candidates from the
  exact distribution beats its mode in every meeting-planning cell (\S\ref{sec:select}).
\end{itemize}

\section{Related Work}

\paragraph{Constrained decoding for diffusion LLMs.}
Exact constrained decoders for absorbing-mask diffusion
models~\citep{austin2021structured,ye2025dream,nie2025llada} cover regular languages (DINGO,
\citealp{suresh2025dingo}), context-free grammars (\citealp{mundler2025cfg}; EPIC,
\citealp{jin2026epic}), lookahead verification (LAVE, \citealp{zhang2026lave}), and finite
automata \citep{dang2026automata}.  The closest work also multiplies mean-field logits by a
graphical model, a sequential automaton, which is the chain case of a factor
graph~\citep{dang2026automata}.  FactorDLM accepts any finite-domain factor graph, so induced
width replaces automaton or grammar size as the cost parameter, and the chain case remains an
automaton plan that compiles together with relational factors.  Proposition~\ref{prop:copy}, the
measured CFG probe (Table~\ref{tab:cfg}), and the grammar-size lower
bound~\citep{filmus2011cfg} quantify the separation.

\paragraph{Other decoding-time methods.}
For autoregressive models, NeuroLogic decoding searches for outputs that satisfy lexical
constraints~\citep{lu2021neurologic,lu2022neurologic}, and GCD tensorizes automata as decoding
proposals~\citep{dang2026gcd}; solver-aided methods such as SatLM hand a declarative
specification to an external solver~\citep{ye2023satlm}, whereas in FactorDLM the model's scores
enter every step as potentials.  For dLLMs, SOAR studies confidence-switched
commitment~\citep{cao2026soar}, and CoDD multiplies the factorized prediction by a trained,
chain-structured probabilistic circuit~\citep{li2026codd}.  Draft-conditioned, soft, and
verifier-guided methods~\citep{reddy2026dccd,tomasi2026primaldual,shao2026cdc,jung2025csp} give
up the exact-support guarantee for broader constraint classes and are complementary.

\paragraph{Exact inference over learned scores.}
Variable elimination and treewidth provide the inference
primitives~\citep{koller2009pgm,dechter1999bucket}; our contribution is a compiler that prices
any declared encoding before allocation, a proof of when direct and sequential encodings
separate, and a plan that serves changing potentials under a declared width.  Knowledge
compilation serves the same repeated-query workload: a compiled circuit answers changing-weight
queries in time linear in its size~\citep{darwiche2002map}, and every width-$w$ elimination plan
embeds in a circuit of size $O(L\,d^{w+1})$~\citep{darwiche2003differential}.  We compile
elimination plans because their width, and with it the decision to refuse, is computable before
allocation, and because buckets run as dense device kernels; beyond the declared width, circuits
and approximate inference such as loopy belief propagation~\citep{koller2009pgm} are the natural
extensions, the latter at the cost of the guarantee.
Constrained conditional models and integer-programming inference also combine learned scores with
declarative hard constraints~\citep{roth2004lp,chang2012ccm}, as posterior regularization does
during learning~\citep{ganchev2010posterior}, but they solve each instance once from fixed scores;
arbitrary-order decoding solves the same topology at every step under changing scores, and
correctness must hold after each commitment.

\paragraph{Benchmark shortcuts.}
Shortcut learning describes models that exploit unintended regularities in place of the intended
capability~\citep{geirhos2020shortcut,gururangan2018annotation}.  We apply the same test to the
benchmarks, asking whether a model-free rule reaches the published metric; exact projection makes
the test clean by removing validity as a confound.

\section{FactorDLM}

\subsection{Constrained mean-field prediction}

Let $x=(x_1,\ldots,x_L)$ be the output slots, where slot $i$ takes values in a finite domain
$\mathcal X_i$.  A domain is the task's label set, for example four answer letters or the digits
1--4, and the dLLM logits are restricted to it; Proposition~\ref{prop:quotient} extends a domain
to the full vocabulary.  At denoising state $x^t$, the dLLM supplies unary logits
$u_i(x_i;x^t)$.  A constraint is a set of log factors $f_a(x_{S_a})$ with scopes
$S_a\subseteq\{1,\ldots,L\}$, and a hard factor is zero on allowed assignments and $-\infty$
elsewhere.  FactorDLM uses
\begin{equation}
 p_F(x\mid x^t,C)=\frac{1}{Z(x^t,C)}
 \exp\left\{\sum_{i=1}^L u_i(x_i;x^t)+\sum_a f_a(x_{S_a})\right\}.
 \label{eq:factor-posterior}
\end{equation}
The model evidence stays in the unary terms, and the factors restrict the support to the declared
relations.  Throughout, \emph{projection} means exact conditioning on this support: we multiply
the mean-field prediction by the hard factors and renormalize.  Exactness is per step: each
proposal is an exact MAP or sample of Equation~\eqref{eq:factor-posterior}, and
\S\ref{sec:select} measures how far the composed multi-step sampler departs from the first step's
distribution.  Auxiliary variables enter the same graph: a deterministic schema automaton becomes
a chain of state variables linked by transition factors, and indicators or accumulators carry
running aggregates.  The compiler admits only auxiliary variables that are functions of the output
slots, so every output has exactly one satisfying extension and the marginal over outputs equals
the mean-field prediction conditioned on feasibility.  Outputs of variable length use a length
bound $L$ and an end-of-text padding symbol, with termination carried by the automaton state.

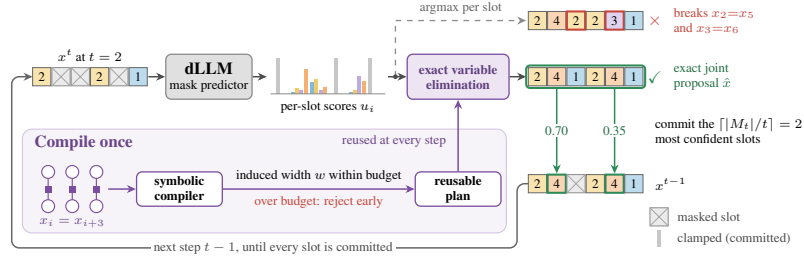
\begin{figure}[t]
\centering
\resizebox{0.8\textwidth}{!}{
\definecolor{fdV1}{HTML}{6BAED6}
\definecolor{fdV2}{HTML}{F2C14E}
\definecolor{fdV3}{HTML}{B39DDB}
\definecolor{fdV4}{HTML}{F4A259}
\definecolor{fdFactor}{HTML}{7B4FA0}
\definecolor{fdValid}{HTML}{2E8B57}
\definecolor{fdReject}{HTML}{C8453B}
\providecommand{\fdtok}[4][]{%
  \filldraw[fill=fdV#4!45, draw=black!60, line width=0.5pt, #1]
    (#2-0.17,#3-0.17) rectangle (#2+0.17,#3+0.17);
  \node at (#2,#3) {#4};}
\providecommand{\fdmask}[2]{%
  \filldraw[fill=black!8, draw=black!45, line width=0.5pt]
    (#1-0.17,#2-0.17) rectangle (#1+0.17,#2+0.17);
  \draw[black!35, line width=0.4pt]
    (#1-0.17,#2-0.17) -- (#1+0.17,#2+0.17) (#1-0.17,#2+0.17) -- (#1+0.17,#2-0.17);}
\begin{tikzpicture}[
    font=\scriptsize,
    >={Stealth[length=1.7mm, width=1.5mm]},
    arr/.style={->, line width=0.8pt, black!70},
    box/.style={rounded corners=3pt, line width=0.7pt, align=center, inner sep=2pt},
]

\node at (1.58,3.8) {$x^t$ at $t=2$};
\fdtok{0.68}{3.4}{2}
\fdmask{1.04}{3.4}
\fdmask{1.40}{3.4}
\fdtok{1.76}{3.4}{2}
\fdmask{2.12}{3.4}
\fdtok{2.48}{3.4}{1}
\node[box, fill=black!12, draw=black!45, minimum width=1.4cm, minimum height=0.84cm]
  (llm) at (3.8,3.4) {{\footnotesize\bfseries dLLM}\\mask predictor};
\draw[arr] (2.67,3.4) -- (llm.west);
\foreach \x/\p/\q/\r/\s in {
    5.48/0.05/0.15/0.10/0.70, 5.84/0.33/0.40/0.10/0.17, 6.56/0.05/0.10/0.50/0.35} {
  \fill[fdV1!85] (\x-0.15,3.08) rectangle (\x-0.08,3.08+0.66*\p);
  \fill[fdV2!85] (\x-0.07,3.08) rectangle (\x+0.00,3.08+0.66*\q);
  \fill[fdV3!85] (\x+0.01,3.08) rectangle (\x+0.08,3.08+0.66*\r);
  \fill[fdV4!85] (\x+0.09,3.08) rectangle (\x+0.16,3.08+0.66*\s);
}
\foreach \x/\o in {5.12/-0.07, 6.20/-0.07, 6.92/-0.15} {
  \fill[black!22] (\x+\o,3.08) rectangle (\x+\o+0.07,3.74);}
\draw[black!45, line width=0.4pt] (4.95,3.08) -- (7.09,3.08);
\draw[arr] (llm.east) -- (4.88,3.4);
\node at (6.02,2.84) {per-slot scores $u_i$};
\node[box, fill=fdFactor!14, draw=fdFactor, text=fdFactor!80!black, minimum width=1.9cm,
      minimum height=0.84cm] (ve) at (8.45,3.4) {{\bfseries exact variable}\\{\bfseries elimination}};
\coordinate (fork) at (7.28,3.4);
\draw[arr] (7.14,3.4) -- (ve.west);
\draw[->, line width=0.7pt, black!45, dashed] (fork) -- (7.28,4.45) -- (9.72,4.45);
\fill[black!70] (fork) circle (0.045);
\node[anchor=south, text=black!60] at (8.5,4.45) {argmax per slot};

\fdtok{9.93}{4.45}{2}
\fdtok{10.29}{4.45}{4}
\fdtok[draw=fdReject, line width=1.1pt]{10.65}{4.45}{2}
\fdtok{11.01}{4.45}{2}
\fdtok[draw=fdReject, line width=1.1pt]{11.37}{4.45}{3}
\fdtok{11.73}{4.45}{1}
\node[text=fdReject, font=\bfseries\small] at (12.12,4.45) {$\times$};
\node[anchor=west, text=fdReject, align=left] at (12.35,4.45) {breaks $x_2{=}x_5$\\and $x_3{=}x_6$};

\fdtok{9.93}{3.4}{2}
\fdtok{10.29}{3.4}{4}
\fdtok{10.65}{3.4}{1}
\fdtok{11.01}{3.4}{2}
\fdtok{11.37}{3.4}{4}
\fdtok{11.73}{3.4}{1}
\draw[fdValid, line width=1.1pt, rounded corners=2pt] (9.72,3.19) rectangle (11.94,3.61);
\draw[arr] (ve.east) -- (9.72,3.4);
\node[text=fdValid, font=\small] at (12.12,3.4) {$\checkmark$};
\node[anchor=west, text=fdValid!80!black, align=left] at (12.35,3.4)
  {exact joint\\proposal $\hat x$};

\draw[arr, fdValid] (10.29,3.18) -- (10.29,1.59);
\draw[arr, fdValid] (11.37,3.18) -- (11.37,1.59);
\node[fill=white, inner sep=1pt, text=fdValid!80!black] at (10.29,2.4) {0.70};
\node[fill=white, inner sep=1pt, text=fdValid!80!black] at (11.37,2.4) {0.35};
\node[anchor=west, align=left] at (12.0,2.4)
  {commit the $\lceil |M_t|/t\rceil = 2$\\most confident slots};
\fdtok{9.93}{1.4}{2}
\fdtok[draw=fdValid, line width=1.1pt]{10.29}{1.4}{4}
\fdmask{10.65}{1.4}
\fdtok{11.01}{1.4}{2}
\fdtok[draw=fdValid, line width=1.1pt]{11.37}{1.4}{4}
\fdtok{11.73}{1.4}{1}
\node[anchor=west] at (12.0,1.4) {$x^{t-1}$};

\draw[->, line width=0.8pt, black!60, rounded corners=4pt]
  (9.76,1.4) -- (9.55,1.4) -- (9.55,0.2) -- (0.12,0.2) -- (0.12,3.4) -- (0.51,3.4);
\node[fill=white, inner sep=1.5pt, text=black!70] at (5.0,0.2)
  {next step $t-1$, until every slot is committed};
\fdmask{12.2}{0.78}
\node[anchor=west, text=black!70] at (12.42,0.78) {masked slot};
\fill[black!22] (12.165,0.22) rectangle (12.235,0.52);
\node[anchor=west, text=black!70] at (12.42,0.37) {clamped (committed)};

\filldraw[fill=fdFactor!6, draw=fdFactor!35, rounded corners=5pt] (0.32,0.48) rectangle (9.35,2.42);
\node[anchor=west, font=\footnotesize\bfseries, text=fdFactor!85!black] at (0.42,2.17)
  {Compile once};
\foreach \x in {0.8,1.25,1.7} {
  \draw[fdFactor, line width=0.6pt] (\x,1.62) -- (\x,0.98);
  \filldraw[fill=white, draw=fdFactor, line width=0.6pt] (\x,1.62) circle (0.12);
  \filldraw[fill=white, draw=fdFactor, line width=0.6pt] (\x,0.98) circle (0.12);
  \fill[fdFactor] (\x-0.065,1.235) rectangle (\x+0.065,1.365);
}
\node[text=fdFactor!85!black] at (1.25,0.67) {$x_i = x_{i+3}$};
\node[box, fill=white, draw=fdFactor, minimum width=1.6cm, minimum height=0.64cm]
  (comp) at (3.25,1.3) {{\bfseries symbolic}\\{\bfseries compiler}};
\node[box, fill=white, draw=fdFactor, minimum width=1.6cm, minimum height=0.64cm]
  (plan) at (8.45,1.3) {{\bfseries reusable}\\{\bfseries plan}};
\draw[arr, fdFactor] (1.92,1.3) -- (comp.west);
\draw[arr, fdFactor] (comp.east) --
  node[anchor=south, text=black] {induced width $w$ within budget}
  node[anchor=north, text=fdReject] {over budget: reject early} (plan.west);
\draw[arr, fdFactor] (plan.north) -- (ve.south);
\node[anchor=east, text=fdFactor!85!black] at (8.35,2.2) {reused at every step};

\end{tikzpicture}}
\caption{\textbf{Overview of FactorDLM} on the copy relation $x_i=x_{i+3}$ (cell colors encode
values).  The same per-slot scores give an invalid assignment under per-slot argmax and a valid
joint proposal under exact elimination over a plan compiled once; the most confident proposed
slots are committed and the next step starts from $x^{t-1}$.}
\label{fig:system}
\end{figure}

\paragraph{Exact inference.}
Variables are eliminated in log space along a greedy min-fill order, or the best of seeded
randomized min-fill orders when greedy is poor.  Max-reduction with backtracking gives a MAP
assignment, sum-reduction gives the partition function $Z$, exact backward sampling, and exact
marginals, and Lawler partitioning over repeated MAP calls gives the exact top-$K$ assignments;
committed slots enter as clamped unary evidence.  If the largest bucket holds $w+1$ variables
with domain size at most $d$, time and memory contain an $O(d^{w+1})$ term, precisely $\exp W$
for the weighted width $W=\max_B \sum_{i\in B}\log|\mathcal X_i|$ over buckets $B$.  Treewidth
is NP-hard, so we report the executed order's width $w$ and the measured peak table per
constraint family.

\paragraph{Symbolic compilation and device-resident execution.}
A symbolic compiler fixes bucket routing and message shapes without allocating a dense message,
rejects any bucket or program over its declared entry budget, and transfers the fixed factors to
the device once, so changing dLLM logits reuse the plan at every step: one elimination sweep and
one backtracking or sampling sweep, at a cost set by the compiled table sizes.  A CPU log-space
engine is the independent numerical reference.  Every invalid assignment violates a hard factor
and has zero probability under Equation~\eqref{eq:factor-posterior}, so exact MAP or sampling
returns a valid assignment whenever $Z>0$; zero mass, malformed inputs, and over-budget tables
raise errors.

\subsection{Denoising and commitment}

Projection returns a complete valid assignment, and the decoder must still decide which positions
to reveal.  Let $M_t$ be the masked slots at step $t$, counting down from $T$ to 1.  Each step
draws an exact MAP assignment or exact sample from Equation~\eqref{eq:factor-posterior}, commits
the $\lceil|M_t|/t\rceil$ proposed values with the highest model confidence, and clamps them as
evidence (Figure~\ref{fig:system}; Algorithm~\ref{alg:decode} in the appendix), so every later
proposal conditions on all commitments and the completed output stays in the support.  In
stochastic decoding each step draws the full proposal by exact backward sampling, exact for that
step's constrained model, and the output composes these per-step conditional samples.
Group-level schedules reduce latency without a replicated semantic gain
(Appendix~\ref{app:schedule-details}).

\subsection{Which encoding is compact}

Every constraint here is a factor graph and an automaton is the chain case, so the question is
which \emph{encoding} to declare: direct, with factors on the output slots, or sequential, with a
chain of auxiliary states.  Consider $k$ four-way answers followed by the same answers in the
same order, the constraint $x_i=x_{k+i}$: its direct encoding is a matching of $k$ equality
factors, with induced width one and a 16-entry peak table.  Over all lengths these strings form
the copy language $\{ww\}$, outside the context-free languages; a grammar covers each fixed
length, but any grammar for the length-$2k$ case has size $2^{\Omega(k)}$~\citep{filmus2011cfg}.

\begin{proposition}[Sequential state against direct encoding]
\label{prop:copy}
(i) Every finite automaton, deterministic or nondeterministic, that recognizes four-symbol
same-order copy of length $2k$ has at least $4^k$ states, so every sequential encoding has state
domains of size $4^k$, while the direct encoding has $k$ factors of 16 entries, total size $O(k)$,
and induced width one.  (ii) Let a relation, after fixing all slots outside $2r$ positions to
feasible values, restrict to exactly $r$ equalities whose endpoints lie on opposite sides of a
position cut; on $q$-value slots every automaton for the full relation then has at least $q^r$
states, while the $r$ equalities are width-one factors.  Conversely, a running aggregate has a
small sequential encoding and a wide direct one: $\sum_i x_i\equiv 0 \pmod q$ needs $q$ states,
and all-different over $n$ slots needs a state for each of the $2^n$ sets of used values, while
their direct encodings are single factors of width $L-1$ and $n-1$.
\end{proposition}
Part (i), a fooling-set argument~\citep{glaister1996fooling} that covers the nondeterministic
automata of the closest prior decoder~\citep{dang2026automata}, separates the encoding classes:
for copy, no sequential encoding is small.  Part (ii) shows that the cheap
encoding differs by constraint, since sequential state counts references that cross a cut while
induced width counts local coupling.  The encoding is chosen per constraint family when the graph
is declared and priced by the compiler before allocation: copy and the JSON relations are direct,
schema syntax and termination are sequential, and all-different in trip planning is where the
direct encoding exceeds the budget (\S\ref{sec:cost}).  Two further results extend exact decoding
to open vocabularies and to outputs whose length the model chooses (proofs and numerical checks
in Appendix~\ref{app:copy-proof}).
\begin{proposition}[Class quotient]
\label{prop:quotient}
If every factor is constant on each class of a partition of the alphabet, then exact MAP, the
partition function, and exact sampling over the full alphabet equal exact inference on the
class-level graph with per-class reduced unaries, followed by expansion within each class.
\end{proposition}
\begin{proposition}[Length on a padded grid]
\label{prop:termination}
Consider templates that differ by inserting one free-text slot at position $p$ before a fixed
suffix, padded to a common length, and let $u_p(\square)$ be the unary of the suffix's first
symbol, the closing delimiter, at slot $p$.  The grid score of the longer template minus that of
the shorter equals the local margin $r_p-u_p(\square)$ between the reduced free-text score and
the delimiter at $p$, plus the net change of the remaining suffix unaries under the one-position
shift, minus the padding unary at the vacated tail slot.
\end{proposition}
The local margin is the stop-versus-continue decision a sequential decoder makes; the grid adds
terms from positions after the span and evaluates all of them as mean-field scores under the
masked context.  We therefore decide termination sequentially, at one position with the committed
context (\S\ref{sec:support}).  Termination itself is a running aggregate, and
Proposition~\ref{prop:copy}(ii) prices it: cheap as sequential state, wide as a direct factor.

\section{Experimental Setup}
\label{sec:setup}

\paragraph{Tasks.}
We evaluate nine relational benchmarks over six constraint classes; Appendix~\ref{app:protocol}
gives prompts, factors, and full protocols.  Three are standard puzzle and graph tasks: the 900
official $4\times4$ Sudoku puzzles of \citet{dang2026automata}, with all-different on rows,
columns, and boxes; the 447 GRAM graph-coloring test graphs, with an inequality per
edge~\citep{baek2026gram}; and the 1{,}483 GRAM $N$-queens boards, of which 110 exceed the entry
budget (\S\ref{sec:cost}).  Two come from NATURAL PLAN~\citep{zheng2024naturalplan}: in meeting
planning (298 records, three to five people), temporal and travel factors with an explicit
\textsc{skip} admit many valid schedules that differ in utility, and in trip planning, flights and
date windows define an exact finite support over the 216 of 800 records whose reference
itinerary the parser recovers.  The other four are constructed with exact model-free floors.
Relational Countdown has one succeeding expression among 96 structural assignments, and its
factors cover slot types and distinct numbers, so success requires arithmetic.  The three JSON
tasks form a progression: referential JSON compiles a schema automaton with relational factors,
variable-length JSON adds two-token identifiers and records of 77 to 104 tokens, and
open-vocabulary JSON adds a note of instructed free text over the model's entire vocabulary.
Two further studies use the same machinery: same-order copy over MMLU
questions~\citep{hendrycks2021mmlu} with growing block length $k$ supplies the scaling study, and
BFCL-derived function-call workflows test relational consistency among supplied candidate calls
under the official executable checker.

\paragraph{Models and hardware.}
We use Dream-7B-Instruct~\citep{ye2025dream} and LLaDA-8B-Instruct~\citep{nie2025llada} as
released (Base checkpoints on the shared Sudoku protocol), and repeat every Instruct task
unchanged on the preference-optimized LLaDA-1.5.  All runs use BF16 on one 48~GB NVIDIA RTX
A6000; CPU timings use a dual-socket Intel Xeon Gold 6338 host (64 cores at 2.0~GHz).

\paragraph{Baselines.}
\emph{Independent decoding} keeps the backbone, prompt, and slot layout and removes projection.
On Sudoku we also run the released CFG~\citep{mundler2025cfg}, EPIC~\citep{jin2026epic}, and
LAVE~\citep{zhang2026lave} decoders on the identical puzzles and budget, a local exact DFA, and
CP-SAT with eight parallel workers; where a task admits model-free solutions we report uniform,
random, and position-shuffled unary controls (\S\ref{sec:audits}).

\paragraph{Reporting.}
Latency covers model calls and constraint inference with device-synchronized timing; paired
comparisons use exact McNemar tests and 20{,}000-resample bootstrap intervals; every protocol and
success criterion was fixed before the models ran, and held-out splits ran once
(Appendix~\ref{app:provenance}).

\section{Results}

Validity is the fraction of outputs that satisfy every declared factor, and overhead is relative
to unconstrained decoding of the same prompt; the subsections ask where exact support holds, what
it costs, what standard benchmarks measure, and how to extract the model's preference.

\begin{table}[t]
\centering
\footnotesize
\setlength{\tabcolsep}{3.8pt}
\renewcommand{\arraystretch}{0.92}
\caption{\textbf{Task accuracy (\%) with projection (\textsc{Ours}) and best-of-8 selection
(\textsc{+Sel}) against unconstrained decoding (\textsc{Base}).}  Bold: best per cell; dashes:
outside the protocol.  $^{\dagger}$Unique completion; the last two rows also
admit model-free shortcuts (\S\ref{sec:audits}).  Floors: $^{\ddagger}$chance 1.0\%, template
13.2\%; $^{\S}$referential chance 0.09\%, template 1.6\%.  Meeting: the three-person split;
$N$-queens: 1,373 of 1,483 boards run, 110 refused before decoding (\S\ref{sec:cost}).
\textsc{+Sel} decodes $K{+}1$ times and scores every rendered token (cost in \S\ref{sec:select}).}
\label{tab:main}
\begin{tabular}{cll rrr rrr rrr}
\toprule
& & & \multicolumn{3}{c}{Dream-7B} & \multicolumn{3}{c}{LLaDA-8B} & \multicolumn{3}{c}{LLaDA-1.5} \\
\cmidrule(lr){4-6} \cmidrule(lr){7-9} \cmidrule(lr){10-12}
& Task & Metric & Base & Ours & +Sel & Base & Ours & +Sel & Base & Ours & +Sel \\
\midrule
\multirow{2}{*}{\rotatebox[origin=c]{90}{\tiny PLAN}}
 & Meeting (3 people) & norm.\ objective & 42.8 & 55.8 & \textbf{75.3} & 62.0 & 83.5 & \textbf{99.0} & 65.0 & 87.7 & \textbf{98.8} \\
 & Trip     & exact match      & 19.9 & \textbf{92.1} & 82.9 &  0.0 & \textbf{84.7} & 72.2 &  0.0 & \textbf{78.7} & 72.2 \\
\midrule
\rotatebox[origin=c]{90}{\tiny MATH}
 & Countdown$^{\ddagger}$ & success & 0.0 & 1.4 & \textbf{3.0} & 0.0 & 1.0 & \textbf{9.4} & 0.0 & 0.8 & \textbf{8.6} \\
\midrule
\multirow{3}{*}{\rotatebox[origin=c]{90}{\tiny JSON}}
 & JSON$^{\S}$ & instructed record & 0.0 & 0.0 & 0.0 & 0.0 & 41.4 & \textbf{92.0} & 0.0 & 43.4 & \textbf{96.4} \\
 & Var-JSON$^{\S}$ & instructed record & 0.0 & 0.0 & \textbf{0.8} & 0.0 & 38.3 & \textbf{41.7} & 0.0 & \textbf{37.9} & 35.4 \\
 & Open-JSON$^{\S}$ & instructed record & 0.0 & \textbf{22.9} & --- & 0.0 & \textbf{26.3} & --- & 0.0 & \textbf{30.0} & --- \\
\midrule
\rotatebox[origin=c]{90}{\tiny PUZZLE}
 & Sudoku$^{\dagger}$ & exact match & 26.9 & \textbf{100.0} & --- & 33.7 & \textbf{100.0} & --- & --- & --- & --- \\
\midrule
\multirow{2}{*}{\rotatebox[origin=c]{90}{\tiny AUDIT}}
 & Coloring & valid coloring &  2.7 & \textbf{100.0} & --- &  0.9 & \textbf{100.0} & --- & 0.2 & \textbf{100.0} & --- \\
 & $N$-queens & valid placement &  0.0 & \textbf{100.0} & --- &  0.0 & \textbf{100.0} & --- & --- & --- & --- \\
\bottomrule
\end{tabular}
\end{table}

\subsection{Exact relational support holds where sequential state explodes}
\label{sec:support}

Validity under projection is guaranteed, so the 100\% column of Table~\ref{tab:main} checks the
implementation on 4{,}516 executed instances for Dream and LLaDA-8B (2{,}243 for LLaDA-1.5, which
skips the two Base-checkpoint tasks), and what the table measures is the gap projection closes:
independent decoding is 0.0--79.0\% valid across nine benchmarks, and exactly 0\% on GRAM
$N$-queens and LLaDA trip planning.

Same-order copy has width one and a 16-entry peak table at every $k$
(Figure~\ref{fig:overview}b), and FactorDLM stays 100\% valid through $k=32$, while the
sequential bound of Proposition~\ref{prop:copy} binds at practical sizes: at $k=10$ the official
CFG engine needs 44.0~MB of grammar, 4.76~s to compile, 8.24~s per query, and 4.11~GiB
(Table~\ref{tab:cfg}).

On the shared 900-puzzle Sudoku protocol, FactorDLM, official LAVE, and our local exact DFA all
return the gold completion on 900/900 for both backbones and decode modes, against
26.89--33.67\% for the matched format+clue control (Figure~\ref{fig:sudoku-sota-quality}); each
grammar decoder receives a single-grid per-puzzle grammar, which favors the grammar engines
(Appendix~\ref{app:protocol}).

Coverage extends beyond equality matching.  On all 447 public GRAM test graphs, independent
decoding is valid on 2.68\% and 0.89\% of graphs (Table~\ref{tab:gram}).  On trip planning, exact
conditioning recovers the reference itinerary on 92.1\% and 84.7\% of records, against 19.9\% and
0.0\% unconstrained, whose valid and correct rates coincide, so the constraints account for the
entire gap.  On function-call workflows, where each call is selected among supplied candidates
that include the gold call, format-only decoding keeps every call well typed yet executes on
1.6\% and 0.8\% of records under the official checker, while relational projection executes on
16.1\% (Table~\ref{tab:bfcl-corrected}), so the failures of format-only decoding are relational.

\paragraph{Syntax and relations in one exact decoder.}
Compiling a schema automaton and relational factors into one graph enforces both at once.
\emph{Referential JSON} asks for a record that declares two users and an action that references
them: the automaton becomes a transition chain, and typed slots, distinct identifiers, reference
resolution, and a role permission become factors at induced width 5, so one forward pass yields
the exact top-8 valid records (Figure~\ref{fig:overview}d) and the plan is 100\% valid on both
axes for every backbone.  \emph{Variable-length JSON} removes the fixed grid: records hold two or
three users with two-token identifiers, valid outputs span 77 to 104 tokens, references resolve
across spans, and the model's own end-of-text token decides termination.  The automaton alone
keeps every record well formed yet emits the wrong number of users on 27--51\% of records and
leaves a broken reference on 64--73\%, while the combined plan is exactly valid at every emitted
length and reaches the instructed record on 35.4--41.7\% for the LLaDA family
(Appendix~\ref{app:protocol}).  \emph{Open-vocabulary JSON} removes the fixed alphabet: the note
field holds instructed free text over the model's entire vocabulary, and
Proposition~\ref{prop:quotient} keeps inference exact.  When the decoder commits the structural
slots before the free text, every backbone copies the instructed phrase on all 240 records at
100\% validity, whereas a single-pass MAP copies it on 0 of 240, consistent with
Proposition~\ref{prop:termination}.  When termination is decided sequentially, the model also
chooses the correct span length on 98.3--99.6\% of records and reaches the instructed record on
22.9--30.0\% (Table~\ref{tab:main}).  These tasks show composition, and we checked the
separation too: the minimal product automaton for schema plus references has 3,021 states on
referential JSON and at least 23,409 on variable-length JSON (a Myhill--Nerode count of declared
identifier-role sets), against hybrid peak tables of 864 and 7.6M entries.  A product automaton
is a practical alternative here; the hybrid plan's gain is that schema and relations are declared
separately and compile together, and the exponential separation is the copy result.

\subsection{Exactness adds at most 6.9\% and stops exactly where declared}
\label{sec:cost}

Relative to unconstrained decoding, exact projection adds 0.41--6.86\% latency across tasks, and
in a balanced 640-output profile, model kernels account for over 98.6\% of latency; this is
projection alone, excluding trip planning's support enumeration and the scoring in
\S\ref{sec:select}.  On matched
Sudoku cells FactorDLM decodes $2.5\times$ faster than official LAVE, the one released system
also at 900/900, and a specialized exact DFA is 1.6--1.8\% faster still, the cost of generality
when a chain suffices.

For repeated queries over a fixed topology, the decoding workload, we compare against CP-SAT on
the eight cells where it proves all 20 changing-unary queries optimal.  All 320 objectives agree
exactly, and compiled CUDA projection is $13.6\times$ faster than CP-SAT with eight parallel
workers, the CPU engine $43.1\times$ (Table~\ref{tab:solver-scaling}); both timings cover the
solve call only, and CP-SAT solves each changed objective from the start.  Rebuilding the
FactorDLM plan per query adds 9--36\%, so the margin comes from variable elimination, and with
both solvers building their model per query FactorDLM remains $7.1\times$ faster.

The width boundary also decides the encoding, and all-different is where the direct one fails:
trip planning couples every position through a permutation, and its direct encoding, verified to
have the same support on all 216 records, runs at width 4--6 up to six cities and exceeds the
budget beyond, as Proposition~\ref{prop:copy}(ii) predicts, so it is declared as an
explicit-support factor listing the feasible itineraries (median 5, at most 108).

Cost grows exponentially with induced width, and Figure~\ref{fig:systems-evidence} reports where
our system stops: of the 140 matrix cells, FactorDLM executes 84 and rejects 56 before numerical
inference, and on GRAM $N$-queens it refuses 110 of the 1,483 boards, whose peak tables would
reach $10^8$--$10^9$ entries; CUDA wins on the width-nine Sudoku graph and the CPU on tiny tables
(Figure~\ref{fig:device-validation}).

\begin{figure}[t]
  \centering
  \includegraphics[width=\linewidth]{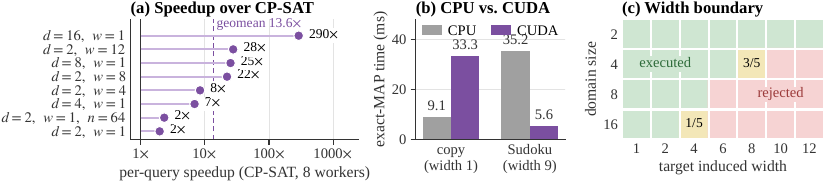}
  \caption{\textbf{Solver comparison and the width boundary.}  (a) Per-query speedup of compiled
  CUDA projection over CP-SAT with eight workers on the eight fully optimal cells ($d$: domain
  size, $w$: induced width, $n=16$ unless shown).  (b) Exact-MAP latency and the CPU/CUDA
  crossover.  (c) Executed and budget-rejected cells.}
  \label{fig:systems-evidence}
\end{figure}

\subsection{Model-free rules solve three of five standard benchmarks}
\label{sec:audits}

\begin{figure}[t]
\centering
\includegraphics[width=0.73\linewidth]{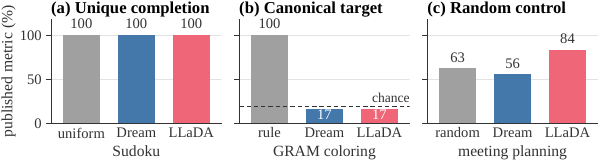}
\caption{\textbf{Three shortcut classes on the published metrics.}  (a) Unique completion:
uniform potentials tie every model at 100.  (b) Canonical targets: a two-line rule hits 309/309
and the $N$-queens target is the lexicographic minimum (405/405); every dLLM method sits below
chance.  (c) The published uniform-MAP control scores zero by tie-break and is omitted; a random
control beats Dream (Table~\ref{tab:audits}).}
\label{fig:shortcuts}
\end{figure}

Exact projection removes constraint violations as a confound, so we can ask what a benchmark
measures.  We test the five standard benchmarks (Sudoku, GRAM coloring and $N$-queens, NATURAL
PLAN meeting and trip planning) for \emph{shortcuts}: model-free rules, canonical targets, or
degenerate baselines that reach the published metric with zero model input.
Three of the five have one (Figure~\ref{fig:shortcuts}).  Sudoku's rules and clues determine a
unique completion, so the constraints solve it with uniform potentials, whereas the weaker
format-plus-clue constraints leave 27--34\% to the model; both GRAM targets are canonical
constructions that a two-line routine reproduces.  On meeting planning the published control is
degenerate; against a proper random control, LLaDA gains about 21 points and Dream falls below
it.

The compact-label interface is one possible cause: on 468 single-digit Countdown instances with
the same constraints and scorer, decoding surface expression tokens is significantly \emph{worse}
than decoding compact labels on all three backbones (paired $p<10^{-3}$), and both stay below the
template floor, so the weak arithmetic is the models'.

\subsection{Selection recovers preference that greedy projection misses}
\label{sec:select}

On meeting planning the factors encode feasibility and the prompt holds the utility, so the
model's choice among valid schedules is what we measure: we draw $K=8$ exact constrained samples,
render each as prose (``\emph{meet Alice at 2:00PM, skip Bob}''), and keep the best under masked
pseudo-likelihood (PLL) or a list-reading judge.

\begin{table}[t]
\centering
\scriptsize
\renewcommand{\arraystretch}{0.9}
\caption{\textbf{PLL selection beats greedy projection in all nine cells}, significantly in six
after Holm correction (exact McNemar).  \emph{Oracle}: best sampled candidate; \emph{Top-8}: PLL
over the exact top-8.}
\label{tab:rerank}
\begin{tabular}{llrrrrrr}
\toprule
& & & \multicolumn{3}{c}{Exact-optimum rate} & & \\
\cmidrule(lr){4-6}
Backbone & People & Greedy & PLL & Judge & Top-8 & Oracle & Normalized objective \\
\midrule
Dream & 3 & 0.120 & 0.490 & 0.320 & \textbf{0.920} & 0.500 & $0.558\to0.753$ \\
Dream & 4 & 0.040 & 0.290 & 0.270 & \textbf{0.640} & 0.400 & $0.578\to0.754$ \\
Dream & 5 & 0.051 & 0.173 & 0.143 & \textbf{0.286} & 0.276 & $0.533\to0.688$ \\
\midrule
LLaDA & 3 & 0.530 & 0.970 & 0.970 & \textbf{0.980} & 0.970 & $0.835\to0.990$ \\
LLaDA & 4 & 0.410 & 0.570 & \textbf{0.730} & 0.450 & 0.750 & $0.711\to0.852$ \\
LLaDA & 5 & 0.133 & 0.214 & \textbf{0.296} & 0.133 & 0.367 & $0.567\to0.698$ \\
\midrule
LLaDA-1.5 & 3 & 0.650 & \textbf{0.970} & 0.950 & 0.940 & 0.980 & $0.877\to0.988$ \\
LLaDA-1.5 & 4 & 0.450 & 0.510 & \textbf{0.730} & 0.500 & 0.770 & $0.753\to0.843$ \\
LLaDA-1.5 & 5 & 0.143 & 0.224 & \textbf{0.337} & 0.245 & 0.398 & $0.604\to0.710$ \\
\bottomrule
\end{tabular}
\end{table}

\textbf{PLL selection beats greedy decoding in every cell} of Table~\ref{tab:rerank}, and the
normalized objective improves everywhere.  Testing the declared scorer alone, exact paired McNemar
tests with a Holm correction over the nine cells leave six significant; the other three, LLaDA at
five people and LLaDA-1.5 at four and five, have adjusted $p\ge0.12$.  On LLaDA with three
people selection reaches $0.990$ against $1.000$ for a solver given the objective, and LLaDA-1.5
replicates the cell ($0.650\to0.970$).  The gain
requires the scorer: greedy is the per-step MAP of the constrained mean-field prediction, a
single sample is \emph{worse} than greedy on both four-person cells, and the reranker stays
within one record of the oracle at every budget (Figure~\ref{fig:frontier}).  Because the judge repeats its choice on 0.43--0.78 of trials under
ten seeded permutations of the candidate list, PLL is the default scorer
(Table~\ref{tab:judge-order}).

Slot-level and candidate-level evidence diverge as the problem grows: at five people the model's
unary potentials are significantly \emph{worse} than uniform ones
(Table~\ref{tab:natural-plan-extended}), yet scoring complete candidates still beats the
constrained argmax.  Candidates are both rendered as text and scored as whole sequences, so
which change carries the gain is open.  Trip planning marks the scope of selection: its oracle
lies within 4.2--10.7 points of greedy, so reranking loses there.

\paragraph{An exact pool from one forward pass.}
Lawler partitioning over the compiled plan enumerates the exact top-$K$ valid assignments of the
first-step distribution, so eight candidates cost one forward pass in place of 27.  Scored
identically (Top-8 in Table~\ref{tab:rerank}), this pool is better in three of nine cells,
indistinguishable in five, and worse in one, and it helps Dream most ($0.490\to0.920$ at three
people, against a sampled-pool oracle of 0.500).  The pools differ because composed multi-step
draws drift from the first-step distribution, to total variation 0.54--0.63 on Countdown against
a sampling floor of 0.13--0.22 (Appendix~\ref{app:protocol}).

\paragraph{What selection costs and what it can be attributed to.}
Greedy projection on meeting planning uses three forward passes and the sampled pool 27; PLL
adds one masked forward pass per rendered token per candidate, on the order of one hundred per
record, so the \textsc{+Sel} columns of Table~\ref{tab:main} cost over an order of magnitude
more model computation than greedy, which we did not time end to end (on referential JSON: one
pass for the pool, $8\times73=584$ for PLL).  Attribution differs by task: a uniform pool of
eight feasible records is correct with probability at most $8\cdot2/2{,}160=0.7\%$, so the
92--96\% JSON result reflects the model's proposals, whereas on meeting planning, with supports
of at most 57 assignments, the uniform-pool control with the same scorer was not run.

\begin{table}[t]
\begin{minipage}[b]{0.47\linewidth}
\centering
\scriptsize
\setlength{\tabcolsep}{3pt}
\renewcommand{\arraystretch}{0.85}
\begin{tabular}{lrrr}
\toprule
 & Dream & LLaDA & LLaDA-1.5 \\
\midrule
\emph{Chance floor}        & \emph{0.010} & \emph{0.010} & \emph{0.010} \\
\emph{Best fixed template} & \emph{0.132} & \emph{0.132} & \emph{0.132} \\
\midrule
Unconstrained              & 0.000 & 0.000 & 0.000 \\
Greedy (MAP)               & 0.014 & 0.010 & 0.008 \\
Rerank, PLL                & \textbf{0.030} & \textbf{0.094} & \textbf{0.086} \\
\midrule
Best-of-8 oracle           & 0.066 & 0.160 & 0.134 \\
\bottomrule
\end{tabular}
\caption{\textbf{Relational Countdown success} over 500 instances; both floors are analytic.}
\label{tab:countdown}
\end{minipage}\hfill
\begin{minipage}[b]{0.49\linewidth}
\centering
\includegraphics[width=\linewidth]{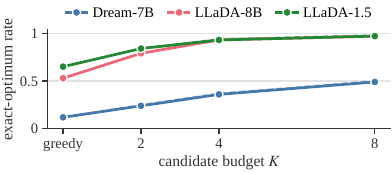}
\captionof{figure}{\textbf{Selection tracks the best-of-$K$ oracle at every budget} (solid:
PLL; dashed: oracle).}
\label{fig:frontier}
\end{minipage}
\end{table}

\paragraph{Benchmarks where preference is necessary and measurable.}
\emph{Relational Countdown} asks for a left-to-right arithmetic expression over three given
numbers, after \citet{dang2026automata}, with both floors analytic over its 96 assignments.  Constrained greedy sits at the chance floor
on every backbone (Table~\ref{tab:countdown}), and selection extracts some arithmetic, yet all
three backbones stay below the template floor.  On referential JSON, selection reaches $0.920$
and $0.964$ against floors of $0.0009$ and $0.016$, so two backbones track references far above
both floors while all of them fail at arithmetic.

\vspace{-4pt}
\section{Conclusion}

FactorDLM decodes diffusion language models under relational constraints by keeping the
constraints as a factor graph and projecting each mean-field prediction onto it exactly.  Its
cost grows with induced width and a sequential encoding's with the state that crosses a position
cut; both are computable before decoding, so each constraint is declared in its cheap encoding
and a schema automaton compiles together with relational factors in one plan that guarantees
validity at under 7\% projection overhead.  Dense constraints such as large $N$-queens boards
exceed the width budget and are refused, and free-text spans are capped at three tokens;
circuits or approximate inference beyond the declared width, and longer spans with sequential
termination, are the next steps.

\bibliographystyle{iclr2027_conference}
\bibliography{references}

\subsection*{AI use statement}

We used generative AI tools to assist with writing: polishing wording and grammar, improving
clarity, and proofreading the manuscript.  The authors reviewed every edit and take full
responsibility for the paper.

\subsection*{Reproducibility statement}

\ifarxiv
The public code release contains the source code, tests, and configurations that run every
experiment and analysis.
\else
The anonymous supplement contains the source code, tests, and configurations that run every
experiment and analysis.
\fi
Section~4 and Appendices~A--D specify the protocol, provenance, full
matrices, nonsignificant comparisons, and failure boundaries.

\appendix
\section{Verification and provenance}
\label{app:provenance}

\paragraph{Exactness validation.}
The inference core is validated against exhaustive enumeration: log partition functions, MAP
assignments, exact top-$k$ sequences, single-variable and component marginals, and empirical
sampling frequencies all match brute force on every tested graph, with numerical error below
$4\times10^{-15}$, and one million constrained draws contain no violation.  The
hybrid compiler is validated the same way: feasibility under the compiled graph equals automaton
acceptance conjoined with the declared relations on every string of the test languages, and
domain propagation leaves the support unchanged.  A CPU log-space engine serves as an
independent numerical reference for the CUDA path, and all 400 cross-device comparisons agree
exactly.  Malformed inputs, impossible evidence, and over-budget constraint graphs raise errors
and the decoder stops there.

\paragraph{Pinned inputs and baselines.}
Every model checkpoint, dataset revision, and external baseline repository is pinned to an exact
revision in the released configurations and baseline lock.  The released CFG, EPIC, and LAVE
baselines are reconstructed from their official repositories at fixed commits, and their upstream
test suites pass in our environments.

\paragraph{Result provenance.}
Every protocol and success criterion was fixed before the corresponding models ran, and held-out
splits were evaluated once.  Each run publishes an immutable output directory containing its
inputs and per-output observations, from which the analyzers compute every aggregate and paired
test.

\paragraph{Statistical units.}
Paired comparisons operate at the record level: 447 graphs for coloring, 900 puzzles for Sudoku,
all planning records per split, 500 instances each for Countdown and referential JSON, and 240
each for variable-length and open-vocabulary JSON, with
exact McNemar tests for paired
binary outcomes and 20{,}000-resample bootstrap intervals elsewhere; grouped designs collapse
sampling seeds within records before testing.  Nonsignificant comparisons are reported alongside
significant ones.

\paragraph{Compute.}
All measurements come from single-GPU processes on 48~GB accelerators, with device-synchronized
timing that separates model calls from constraint inference and excludes checkpoint loading and
file output.  The complete six-method Sudoku matrix costs 5.9 and 6.9 GPU-hours for the two
backbones; the remaining benchmark suites are smaller.

\paragraph{Released artifacts.}
\ifarxiv
The public code release contains the full source:
\else
The supplementary material contains the full source:
\fi
the exact inference core, the hybrid
syntax-relational compiler, the three constructed benchmarks and their generators, every audit as
a standalone routine, and the commands that run every experiment, analysis, table, and figure
in this paper.

\section{Full experimental protocol}
\label{app:protocol}

This appendix records the per-task protocol in full, including the pinned data and checkpoint revisions.  Section~\ref{sec:setup} summarizes what a reader needs to interpret the results.

\paragraph{Models and hardware.}
We evaluate the official Dream-v0-Instruct-7B and LLaDA-8B-Instruct checkpoints at pinned
revisions; every revision in this appendix is listed in the released configurations.  Dream
logits are
aligned with its previous-position prediction convention; LLaDA logits are read at the masked
position, matching its official generator.  Unknown model types fail instead of selecting a
fallback.  Each run uses BF16 on one NVIDIA RTX A6000 (48~GB, compute capability 8.6) in a host
with two Intel Xeon Gold 6338 processors (64 cores at 2.0~GHz base, 3.2~GHz boost) and 1~TB of
memory; all CPU-side measurements, including every CP-SAT solve and the CPU elimination engine,
run on this host.  CUDA is synchronized immediately before
and after each decoded output.  Latency includes model calls and constraint inference, but excludes
checkpoint loading, tokenization, file output, and plotting.
The shared Sudoku study uses the corresponding Base checkpoints, also pinned.

\paragraph{Methods.}
\emph{Independent} commits unconstrained per-slot argmax proposals by confidence.  \emph{Finite
support} enumerates every valid assignment, normalizes the mean-field product over this support,
and commits by exact posterior confidence.  It is an exact probabilistic control on feasible cases;
it is not the optimized runtime of \citet{dang2026automata}.  A declared 100,000
state cap stops compilation before model execution.  \emph{Individual factor} ranks singleton
commitments with exact constrained marginals.  \emph{Components model/factor/random} keep the true
  relation groups but change their ranking signal; \emph{matched random groups} preserves only the
  group-size multiset.  \emph{Dispersed model}, the revised full system, preserves the same atomic
  batch sizes but separates factor-dependent variables across batches and ranks batches by model
  confidence.  Dispersed factor changes only the ranking signal.  Static and residual separators are
  connected-graph controls.

We also compose the published SOAR position-search rule~\citep{cao2026soar} with the exact finite
support.  The independent implementation uses the official defaults: confidence threshold 0.90
for Dream and 0.95 for LLaDA, beam size two, and at most five parallel commitments.  The official
repository has no root license file, so we do not copy its source.  This comparison was committed
before its outputs were inspected, but after the original finite-support held-out result; we label
it as a stronger added baseline rather than the original primary test.

\paragraph{GRAM graph coloring.}
We use all 255 eight-vertex and 192 ten-vertex graphs in the public GRAM test split at its pinned
revision~\citep{baek2026gram}.  The prompt lists graph edges and requests one of three
colors per vertex.  Pairwise inequality factors encode edges.  We run two denoising steps, selected
in the preceding 24-graph feasibility study.  Metrics are valid coloring, exact target partition up
to global color renaming, best-permutation vertex agreement, and latency.  Since many graphs have
multiple valid colorings, validity and agreement are both reported.

\paragraph{GRAM $N$-queens completion.}
The public GRAM test split supplies 1,483 partial boards over the $8\times8$ and $10\times10$
sizes.  Each free row is one variable whose domain is the board columns; clue rows fold into hard
unary factors, and pairwise factors forbid shared columns and diagonals, so a board with $k$ free
queens has induced width $k-1$.  The compiler's declared entry budget rejects, before any
numerical inference, exactly the $10\times10$ boards with eight or nine free queens: 110
instances whose symbolic peak tables reach $10^{8}$ and $10^{9}$ entries.  The remaining 1,373
instances per backbone execute under all six method-and-mode conditions (8,238 of the 8,348
declared rows; each rejection is recorded once), and every validity figure in the paper uses the
executed set as its denominator.  Encoding correctness was verified before execution: the factor
graph's exact log partition reproduces GRAM's independently computed solution counts on 120
records across both board sizes.

\paragraph{MMLU same-order copy.}
Each prompt contains public four-choice MMLU questions and requests two identical output blocks.
The hard constraint sees only $x_i=x_{k+i}$, never the answer key, and surface labels are independently
permuted by a fixed hash.  The 240-question overlap and 1,200-question subject-disjoint causal
matrix measure validity, answer accuracy, pairwise task success, and latency.  The task measures
relational semantic preservation over MMLU questions~\citep{hendrycks2021mmlu}, under its own
protocol rather than the MMLU benchmark's.

\paragraph{Causal attribution matrix.}
The unary block compares model, uniform, random, shuffled-position, and oracle-biased logits.  The
schedule block fixes the exact projection step and compares individual positions, matched random groups,
true components, separators, dispersed batches, and SOAR.  MAP and exact sampling run at 32 steps
with three seeds.  All paired gates were frozen before either held-out output existed.

\paragraph{NATURAL PLAN partial utility.}
We use the official meeting-planning JSON at its pinned commit~\citep{zheng2024naturalplan},
adding an explicit \textsc{skip} value per person so that
factor-valid assignments differ in utility.  Pair factors encode temporal order and directed travel
but not the gold plan.  The 45-cell matrix reports
validity, normalized scheduled-meeting objective, exact optimum, classification, and latency under
model and controlled unary potentials.

\paragraph{NATURAL PLAN trip planning.}
Flight connectivity and date windows are high-arity, so itineraries are held as exact finite
support: an adapter parses flight lists and date windows from the prompt text, enumerates every
feasible itinerary, and the decoder normalizes the mean-field product over that support.  The
evaluated set contains every usable record with four to seven cities: 23, 49, 70, and 74
records, respectively, for a total of 216.  Of the 200 records per size, the parser recovers 44
to 85 and discards every record whose reference itinerary falls outside the reconstructed
support.  All methods run on this identical parseable subset, so method comparisons stay
unbiased, while absolute rates describe the subset and may differ on the full distribution.  Enumerated supports have median 5, mean 11.9, and maximum 108 itineraries;
the chance floor is the mean reciprocal support size, 0.404, 0.319, 0.230, and 0.156 by size and
0.243 pooled.

The factored alternative encodes the same constraint as a graph: position variables hold the
city, a chain of day variables holds the running start day (pinned at day one), ternary factors
advance the chain by the placed city's duration and apply that city's windows, and pairwise
factors enforce all-different and connectivity.  Its support matches the enumerated support on
every evaluated record, counted by exact partition under uniform potentials, and a unit test
checks set equality directly on a three-city instance with and without windows.  All-different
over positions forces induced width $n-1$, so realized peaks are 1,792--36,864 entries at four
cities, 25,000--105,125 at five, and 466,560--1,026,432 at six, against the declared 1,048,576
budget; every seven-city record is rejected before allocation.  The finite-support plan is
therefore the affordable exact encoding at the evaluated sizes, and both encodings condition on
the identical set.

\paragraph{BFCL-derived relational workflows.}
We pin the official Gorilla source and BFCL v4 base multi-turn data at a fixed
commit.  The 124 records have two through eight calls and a repeated scalar literal.
Factors enforce only cross-call consistency, so several plans remain valid; the official checker
measures execution.  We compare independent, one-shot CP-SAT, iterative schedules, and semantic
verification.  Each candidate check receives isolated mutable tool state.  This controlled
derivative is not an official BFCL leaderboard submission.

\paragraph{Official Sudoku shared task.}
We pin the nine 4-by-4 Sudoku JSONL files in the official Dream repository at a fixed commit.
For each clue count 4 through 12, the
first eight records are demonstrations and the remaining 100 are test puzzles, exactly matching the
900-puzzle protocol of \citet{dang2026automata}.  \emph{Format+clue} restricts every
slot to digits 1--4 and keeps input clues fixed, matching the semantics of their 21-state DFA; it
does not enforce Sudoku relations.  \emph{Factor separator} additionally uses 56 pairwise
inequality factors for all rows, columns, and boxes.  Its min-fill width is nine and its declared
limit is 1,048,576 entries.  The exact-DFA control compiles all 288 valid 4-by-4 grids into a
502-state, 656-edge automaton and applies clues as observed evidence; it is an independent exact
implementation, not the unavailable authors' code.  All 900 puzzles run greedy and fixed-seed
exact sampling at the published 32 steps; a fixed 90-puzzle subset (10 per clue count) measures
greedy 4/8/16-step scaling.  We report clue preservation, valid-Sudoku rate under the official
metric, exact match, synchronized decoding latency, CUDA model time, and non-model residual.
Published automaton accuracies are shown separately as a non-executed reference.

\paragraph{Executable CFG-system overlap.}
We also run the locked official CFG decoder~\citep{mundler2025cfg}, LAVE~\citep{zhang2026lave}, and
EPIC~\citep{jin2026epic} on the identical Sudoku cells.  An independent adapter enumerates the 288
rule-valid grids once and filters them using only the observed clues; the reference answer is used
only for scoring.  The resulting per-puzzle finite language is passed through each system's
official grammar interface.  All 900 puzzles have exactly one feasible completion, so
each per-puzzle grammar accepts a single string, the smallest exact grammar for the task, and the
comparison favors the grammar engines; the adapter still records the compatible-language size of
every puzzle and assumes nothing about uniqueness.
The shared overlap emits exactly 16 digit tokens, matching FactorDLM's fixed-slot output budget;
it is not a reproduction of the automaton paper's separate 32-token free-text configuration.
For LAVE, we use its released $N=10$, top-5 proposal, and $\tau=5$ retry settings.  For EPIC, we use
the released lexing-cache, DFA-free, and exact regular-cover configuration.  LAVE's official
generators insert EOS one position after a grammar match, so its adapter allocates one transport-only
EOS slot and excludes it from the 16 scored semantic slots.  Dream-Base has two added model-only
token IDs outside the pinned Qwen checker vocabulary; the adapter masks exactly those
checker-unrepresentable logits, which cannot be grammar terminals.  The comparison uses
both Base backbones, MAP and temperature-0.2 sampling at 32 steps, and the same 90-puzzle MAP sweep
at 4/8/16 steps.  Each raw row separates synchronized decode time, per-instance grammar/checker
setup, their sum, retries, grammar bytes, and feasible-language size.  Static FactorDLM and DFA
plans are compiled once and reused; their per-output setup is therefore zero in the amortized
serving comparison.  Finite-budget completion failure remains a failed output: the adapter never
repairs it or invokes another method, and it retains the raw decoded token string for audit.  This
is the recovery-disabled setting denoted Con.$^{-}$ by the CFG paper and E.$^{-}$ by EPIC; both
papers also report a separate completion-based repair mode.  We therefore compare direct fixed-budget
outputs and do not reinterpret this matrix as superiority over their repaired modes.

\paragraph{Official CFG engine probe.}
We pin the official constrained-diffusion repository~\citep{mundler2025cfg} at a fixed
commit.  Its unmodified Rust grammar engine passes 406 upstream tests with eight
declared skips.  For each $k\in\{2,4,6,8,10\}$, we enumerate the finite same-order-copy language into
a CFG, then measure grammar bytes, compilation, one membership query, and process peak RSS.  This
is a direct engine measurement of the finite restriction, not a claim that enumeration is the only
way to encode an arbitrary finite language.

\paragraph{Evaluation audits.}
Three audits test whether a task metric is fixed by something other than model reasoning, and each
is published as a runnable module.  The Sudoku solver audit replaces
every non-clue unary with a uniform potential and re-solves, establishing whether the constraints
alone determine the answer.  The N-Queens target audit computes the exact row-order
lexicographic-minimum completion by weighted MAP over the same factor graph (weighting free row
$j$ by $-c\cdot n^{k-1-j}$ makes the MAP assignment the lexicographic minimum) and compares it
with the published target on every multi-solution instance inside the table budget.  The planning
audit regresses each method-and-mode cell's mean normalized objective on its mean \textsc{skip}
rate, excluding model-free methods and the degenerate uniform-MAP cell, and reports the residual as
the part of the score not explained by propensity to schedule.

\paragraph{Extended planning protocol.}
We extend the meeting-planning task from three to four and five people using the same eleven
methods, both proposal modes, three replicate seeds, and three commitment steps, so people count is
the only manipulated variable.  Compact labels extend from digits 1--9 to a 61-symbol alphabet
verified single-token on all four checkpoints, with indices 0--8 unchanged so three-person prompts
stay byte-identical.  A declared 1,048,576-candidate budget rejects a record before its feasible set
is enumerated, retaining 100 of 100 four-person and 98 of 100 five-person records.  Gates were
frozen before execution and name their decode mode explicitly.

\paragraph{Statistics.}
MMLU questions share one generated output inside each group, so questions are not independent.
The primary interval resamples whole groups within the $k=4$ and $k=8$ strata for 20,000 bootstrap
draws.  The primary two-sided test flips the sign of each group effect, also with 20,000 fixed-seed
draws in the legacy analysis and 100,000 in the causal revision.  We report question-level exact
McNemar tests only as descriptive supporting analyses.  GRAM uses paired graph-level tests.  Revised
planning and workflow analyses first average stochastic seeds within each record, then use 20,000
paired record bootstraps and 100,000 sign flips.  Sudoku uses exact paired McNemar tests over all 900
puzzles separately for greedy and sampling.  All methods and nonsignificant comparisons remain in
the consolidated CSV.

\paragraph{Controlled solver and width matrix.}
Path-power primal graphs realize every target width in
$n\in\{16,32,64,128,256\}$, $d\in\{2,4,8,16\}$, and
$w\in\{1,2,4,6,8,10,12\}$.  Five warmups precede 20 changing-unary queries.  Compiled NumPy and
float64 CUDA variable elimination are compared with deterministic one-worker OR-Tools CP-SAT on a
fixed 48-cell subgrid.  We report compile/build, repeated projection/solve, one-shot time,
score agreement, peak and aggregate entries, CUDA allocation, conflicts, branches, and every
per-table or aggregate-budget rejection.  Identical query repeats are paired for solver speedups;
timeouts retain their solver status.

\paragraph{Controlled runtime profile.}
Four NATURAL PLAN records are selected from factor metadata to span peak elimination tables of
8, 27, 64, and 125 entries.  For each backbone and method, we run two warmups and 20 measured
repeats per record.  A deterministic rotation balances every method equally across the four ordinal
positions.  CUDA events measure model-kernel time; synchronized wall time minus event time is the
additive non-model residual.  CUDA peak allocation is reset before every output.  A hardware
check requires an RTX A6000 with compute capability 8.6 and at least 48~GB of memory.

\paragraph{Compiled-system validation.}
We run 100 balanced repetitions of CPU and CUDA MAP inference for a 128-variable width-one copy
graph and the 16-variable width-nine Sudoku graph, after 20 warmups.  Assignments and MAP scores
must agree across every trial.  A separate 1,000-trial CPU study compares dynamic topology
construction with a reused compiled plan at 8--128 variables.  Binary cliques of sizes 4 through 24
test the declared 1,048,576-entry allocation boundary.  Every publisher requires a clean Git tree
before and after execution.

\paragraph{Relational Countdown.}
Instances are generated deterministically from seed 20270: three distinct integers in 1..30, the
target chosen as the reachable value with the fewest witnessing expressions (ties toward the
smallest value, targets equal to a given number excluded), accepting a triple only when that best
target has exactly one witnessing assignment.  All 500 released instances therefore have exactly
one succeeding expression, and the chance floor is exactly $1/96$ on every instance; the
released generator enforces this condition, and the enforced generator reproduces the released
instance set exactly.  Evaluation is strictly left to right; intermediates must be positive
integers and division exact.  The slot layout is number, operator, number, operator, number over a
shared four-value domain, with hard unary type factors and pairwise all-different factors on the
number slots.  Renderings show the fully parenthesized expression and never its value, because a
shown value would let string comparison against the target replace the scorer.  One run per
backbone records every scorer on identical pools ($K=8$, decode seed 17, greedy included).
Both floors were computed before execution: chance $1/96$ per instance and $0.132$ for
the best single input-agnostic assignment (the best fixed template) over the instance set.

\paragraph{Referential JSON.}
Instances are generated deterministically from seed 20270: two distinct identifiers drawn from
six single letters, two distinct roles, and an operation the actor's role permits, so every
instruction is feasible.  The output grid has 73 slots over 27 symbols, each verified
single-token on all three checkpoints; the schema automaton has 74 states and fixes every
structural slot, while seven type unaries, an all-different factor on the declared identifiers,
two reference-resolution factors, and two role-permission factors constrain the seven content
slots.  Forward reachable layers restrict token and state domains before compilation, which is
arc consistency and preserves the support; the weighted min-fill order then realizes induced
width 5 and an 864-entry peak table under declared budgets of 4{,}194{,}304 per table and
33{,}554{,}432 in total.  The feasible set is instance-independent and enumerated exactly at
2{,}160 records, each instance has exactly two correct records (the two declaration orders, both
accepted), so chance is $2/2{,}160$, and the best fixed template over the 500 instances is
$0.016$.  One forward pass on the fully masked grid supplies the logits for every method; the
prompt carries one fixed format demonstration, identical across methods and instances.

\paragraph{Variable-length JSON.}
The record holds two or three users, each with a two-token identifier (a letter and a digit),
and the action object may carry a trailing note, so valid records span 77 to 104 tokens over a
34-symbol alphabet.  A 132-state token-level automaton accepts exactly the four templates plus
trailing padding, padding maps to the tokenizer's end-of-text identifier so that record
termination is priced by the model's own belief that the answer has ended, and a reference
resolves through per-user binary match indicators with a two-state accumulator chain, keeping
every relational factor at arity four or less.  Realized peaks are 54,432 table entries at two
users and 7,620,480 at three under a declared 16,777,216 budget, elimination orders come from
the randomized order search, and the exact top-8 branches only over content positions, which is
exact because states, padding, and indicators are functions of them.  Scoring parses the
rendered text with a JSON parser, and 240 seeded instances mix both sizes.  A first run padded
with an arbitrary symbol instead of end-of-text; the padding was corrected and the run repeated.

\paragraph{Open-vocabulary JSON.}
The record's note value is free text over the model's entire vocabulary, one to three tokens,
while every other field keeps its guarantees.  Exactness survives through a class quotient: the
automaton and every factor depend on a token only through its class, so reducing the
full-vocabulary logits over the open class per slot, eliminating over the class-level graph,
and expanding the winner back to its best token computes the exact optimum over the full space.
The open class is every identifier whose decoded text contains no double quote, backslash, or
control character (146,922 of Dream's 152,064), so any completion renders as valid JSON and the
closing quote stays deterministic.  Three mechanisms were found on the way, each by a traced
probe: a single masked forward gives diffuse marginals whose
full-vocabulary argmax is a filler token; mean-field token scores cannot price open-span
termination, so the note length is declared from the instructed phrase's own tokenization; and
a masked open slot draws a confidently wrong reply-prefix token before structure exists, so the
iterative decoder commits closed slots first and fills open slots from the final forward pass
over the completed structure.  Under that schedule all three backbones copy the instructed
phrase on 240 of 240 records, and Dream records its first nonzero semantic result on a JSON
task (0.200 against 0.000 for its own single-pass MAP), consistent with the drift study's
finding that Dream's composition benefits from departing from its step-$T$ head.

A free-termination variant removes the declared note length.  It decodes over the
union automaton of all three span lengths: closed slots before the note commit first, the note
is then walked left to right with one forward pass per position, and each position commits the
closing quote when its log-probability exceeds that of the best free-text token, and the best
free-text token otherwise.  The remaining structural tail is forced and commits by exact MAP.
Every realized branch lies in the automaton's language, so the validity guarantee is unchanged,
and validity is 100\% on both axes for every record.  The note is copied exactly, with the correct
length, on 0.996 (Dream), 0.983 (LLaDA-8B), and 0.983 (LLaDA-1.5) of records, against 1.000 for
the declared-length decoder run alongside.  Proposition~\ref{prop:termination} explains the
difference: a joint MAP over the padded grid compares lengths through the padding score, while
the sequential decision compares the model's own closing-quote and content probabilities.

\paragraph{Exact top-$k$ selection.}
Lawler partitioning over the compiled plan returns the $k$ highest-scoring valid assignments in
score order: each popped solution splits its region into disjoint children by forcing a growing
prefix and banning the branch value, so every returned assignment is the exact maximizer of a
disjoint region.  The implementation is verified against brute-force enumeration on random
graphs at four pool sizes.  On meeting planning the pool is the exact top-8 of the step-$T$
constrained distribution from one forward pass, rendered and scored by the same
pseudo-likelihood scorer as the sampled pipeline; eleven records per three-person cell have
supports below eight and are therefore exhaustive.

\paragraph{Pseudo-likelihood scorer.}
A candidate rendering $y=(y_1,\ldots,y_n)$ is appended to the prompt $c$ and scored by the
standard masked-language-model pseudo-log-likelihood~\citep{salazar2020mlm},
\begin{equation*}
 \mathrm{PLL}(y\mid c)=\sum_{j=1}^{n}\log p_\theta\bigl(y_j \mid c,\, y_{\setminus j}\bigr),
\end{equation*}
where $y_{\setminus j}$ is the rendering with position $j$ replaced by the mask token and every
other position visible.  Each term takes one forward pass with a single masked position, read at
the model's official logit position, and the sum is unnormalized.  Masking is necessary because
both backbones are bidirectional: an unmasked logit attends to the token it scores.  The
candidate with the largest PLL is returned.

\paragraph{Composed-sampler drift.}
Per-step exactness composes into a distribution no single step defines, and Countdown's 96
enumerable assignments make the gap measurable: for the first 32 released instances, the
step-$T$ constrained distribution $p_T$ is computed in closed form from the first forward pass,
256 composed draws follow the recorded multi-step protocol, and 256 matched draws from $p_T$
supply the finite-sample floor for every statistic.  The floor draws' mean $\log p_T$ sits
within three standard errors of $-H(p_T)$ on both backbones, validating the pipeline.  Composed
draws reach total variation 0.537 (Dream) and 0.628 (LLaDA-8B) against floors of 0.126 and
0.216, an identical excess of 0.412, with opposite directions: Dream's draws land 0.73 nats
below the entropy-matched expectation of $\log p_T$ and LLaDA's 0.45 nats above, so composition
disperses one backbone away from its step-$T$ head and concentrates the other on it, matching
which pool wins in the exact-pool study.

\paragraph{Uncompiled elimination control.}
On the eight cells where CP-SAT proves every query optimal, the identical queries run with the
symbolic plan, the device plan, and the factor transfer rebuilt on every query.  Every one of
the 640 measured trials must reproduce the compiled CPU score exactly; the amortization factor
is the ratio of mean uncompiled to mean compiled latency.

\paragraph{Judge order probe.}
Candidate pools of the matched scorer ablation are reconstructed with identical decode seeds and
judged under the recorded sorted order and ten seeded permutations (order seeds 31 through 40,
fixed before execution).  Agreement is the fraction of permuted trials that repeat the
sorted-order choice, the permuted optimum is averaged over the ten orderings, and the
symmetrized score averages the eleven aligned log-score vectors.  All six Dream and LLaDA-8B
meeting cells are probed, plus the LLaDA-1.5 three-person cell.

\section{Prompt and rendering examples}
\label{app:listings}

One complete example per task family, produced verbatim by the released task modules.  Every
prompt states the task, the compact slot code, and the per-slot value legend; the decoder reads
logits only at the slot positions, restricted to the legend tokens.  Renderings are the prose
forms scored by the selection stage; the Countdown rendering deliberately omits the computed
value (\S\ref{sec:select}).

\subsection{Relational Countdown}

\begin{footnotesize}
\begin{verbatim}
Numbers: 6, 2, 8.  Target: 11.
Combine the three numbers into one arithmetic expression that
equals the target.  Use each number exactly once.  The expression
is evaluated strictly left to right with no operator precedence,
every intermediate value must be a positive whole number, and
division must leave no remainder.

Use the compact expression code below.
Number codes: 1=6, 2=2, 3=8.  Operator codes: 1=+, 2=-, 3=*, 4=/.
Return exactly 5 digits: first number code, operator code, second
number code, operator code, third number code, with no spaces or
explanation.
\end{verbatim}
\end{footnotesize}

\noindent Candidate rendering scored by the selection stage: \texttt{((6 + 2) * 8)}.

\subsection{Meeting planning (prompt tail)}

\begin{footnotesize}
\begin{verbatim}
CONSTRAINTS: You arrive at Bayview at 9:00AM. Ronald will be at
Alamo Square from 8:30AM to 7:45PM. You'd like to meet Ronald for
a minimum of 90 minutes. Richard will be at Union Square from
2:30PM to 9:45PM. You'd like to meet Richard for a minimum of 30
minutes. Kenneth will be at Golden Gate Park from 10:00AM to
3:15PM. You'd like to meet Kenneth for a minimum of 60 minutes.

Use the compact partial-schedule code below.
Return exactly one digit per person in the listed order.
Use 1 to skip a meeting when it cannot or should not be scheduled.
Return no spaces, punctuation, or explanation.
Position 1, Ronald: 1=SKIP, 2=11:10AM, 3=9:16AM, 4=3:15PM.
Position 2, Richard: 1=SKIP, 2=2:30PM.
Position 3, Kenneth: 1=SKIP, 2=10:55AM, 3=10:00AM.
\end{verbatim}
\end{footnotesize}

\noindent Candidate rendering: \texttt{Plan: meet Ronald at 9:16AM, meet Richard at 2:30PM,
meet Kenneth at 10:55AM.}  The judge scorer lists such renderings as a numbered set and reads
the label distribution at one masked answer position.

\subsection{Trip planning (prompt tail)}

\begin{footnotesize}
\begin{verbatim}
Find a trip plan of visiting the cities for 15 days by taking
direct flights to commute between them.

Use the compact itinerary code below.
City codes: 0=Prague, 1=Lyon, 2=Barcelona, 3=Helsinki.
Return exactly 4 digits giving the cities in visiting order, with no
spaces or explanation.
\end{verbatim}
\end{footnotesize}

\noindent Candidate rendering: \texttt{Trip plan: days 1-3 in Prague, then days 3-8 in Lyon,
then days 8-13 in Barcelona, then days 13-15 in Helsinki.}

\section{Commitment and evaluation details}

\subsection{The decoding loop}

\begin{algorithm}[ht]
\footnotesize
\caption{FactorDLM decoding.  \textsc{Eliminate} is exact variable elimination along the
precompiled elimination order: max-product with backtracking for MAP, or sum-product with
backward ancestral sampling for an exact draw from Equation~\eqref{eq:factor-posterior}.}
\label{alg:decode}
\begin{algorithmic}[1]
\Require prompt $c$; slots $1{:}L$ with domains $\mathcal X_i$; hard factors $\{f_a\}$; steps $T$
\State compile the elimination plan once; reject any bucket over its entry budget before allocation
\State $x^T \gets$ all slots masked
\For{$t = T, \ldots, 1$}
  \State $M_t \gets$ the masked slots of $x^t$
  \State $u_i(\cdot) \gets$ dLLM logits at slot $i$ restricted to $\mathcal X_i$
         \Comment{one forward pass}
  \State $u_i(v) \gets -\infty$ for every committed slot $i$ and every $v \neq x^t_i$
         \Comment{clamp evidence}
  \State $\hat x \gets \textsc{Eliminate}(u, \{f_a\})$
         \Comment{jointly valid by construction}
  \State $c_i \gets \operatorname{softmax}(u_i)[\hat x_i]$ for every $i\in M_t$
         \Comment{model confidence}
  \State $A_t \gets$ the $\lceil |M_t|/t \rceil$ slots of $M_t$ with the largest $c_i$
  \State $x^{t-1}_i \gets \hat x_i$ for $i\in A_t$; $x^{t-1}_i \gets x^t_i$ for $i\notin M_t$;
         $x^{t-1}_i \gets \texttt{[MASK]}$ for $i\in M_t\setminus A_t$
\EndFor
\State \Return $x^0$
\end{algorithmic}
\end{algorithm}

\subsection{Proofs of Propositions~\ref{prop:copy}, \ref{prop:quotient},
and~\ref{prop:termination}}
\label{app:copy-proof}

\begin{proof}
(i) Consider the pairs $(w,w)$ for $w\in\{1,2,3,4\}^k$.  Each concatenation $ww$ is in the
language, so fix one accepting run per $ww$ and let $s_w$ be the state that run occupies after
reading the first half.  If the automaton has fewer than $4^k$ states, then by the pigeonhole
principle two distinct prefixes $w\neq v$ share $s_w=s_v$; splicing the first half of $w$'s run onto the second half of $v$'s
run yields an accepting run for $wv$, which is not in the language.  At least $4^k$ states are
therefore required, deterministic or not; the pair set is a fooling set in the sense
of~\citet{glaister1996fooling}, and the argument is identical for any alphabet size $q$.  The
factor graph is a matching of $k$ equality edges, each a $4\times4$ table of 16 entries;
eliminating either endpoint of every edge creates no fill, so the induced width is one and the
peak elimination table has 16 entries.

(ii) Fix the other slots to feasible values and restrict attention to the $2r$ equality slots;
by hypothesis the relation restricted to them is exactly the $r$ equalities, so every $w\in[q]^r$
placed on the cut-left endpoints and repeated on the matched cut-right endpoints completes to a
member of the full relation.  These pairs $(w,w)$ form a fooling set exactly as in part (i): for
$w\neq v$ the splice $wv$ violates some equality, while shared states at the cut would construct
an accepting run for it.  Because the fooling strings lie in the full relation and the splices do
not, any automaton for the full relation has at least $q^r$ states, deterministic or not.  Each
equality is one $q\times q$ factor between its endpoints; the $r$ factors form a matching, so the
induced width is one.  For the aggregate, the automaton tracking the running sum modulo $q$ has
$q$ states; as a factor graph it is the chain $s_i=(s_{i-1}+x_i) \bmod q$ with ternary transition
factors of $q^3$ entries and width two, the sequential encoding the compiler uses for
termination, while the direct encoding is one factor over all $L$ variables with elimination
width $L-1$.  All-different is analogous: the automaton tracks the set of used values, $2^n$
states, while the direct encoding by pairwise inequalities has width $n-1$.
\end{proof}

\begin{proof}[Proof of Proposition~\ref{prop:quotient}]
Fix a class assignment $y$.  Every factor term is constant over the token assignments $x$ with
$c(x_i)=y_i$, so the maximum (respectively the sum) of $\exp\{\sum_i u_i(x_i)+\sum_a
f_a(x_{S_a})\}$ over that set factorizes into the factor terms at $y$ times the per-slot maxima
(respectively sums) of $\exp u_i$ within class $y_i$.  Maximizing or summing over $y$ with the
reduced unaries therefore equals the full-alphabet maximum or partition function, the joint
distribution factorizes as $p(y)\prod_i p(x_i\mid y_i)$ with $p(x_i\mid y_i)$ the unary softmax
restricted to class $y_i$, and the MAP expands each winning class to its within-class argmax.
All three identities are additionally verified numerically against brute-force enumeration on
random instances in the released tests.
\end{proof}

\begin{proof}[Proof of Proposition~\ref{prop:termination}]
Adjacent templates share their prefix and their closed content, so those unary terms cancel in
the score difference.  Let the suffix $S_1\cdots S_m$ start at slot $p$ in the shorter template,
with $S_1=\square$ the closing delimiter, and at $p+1$ in the longer one.  What remains is the
reduced free-text score $r_p$ at the inserted slot, plus $\sum_{j=1}^{m}u_{p+j}(S_j)$, minus
$\sum_{j=1}^{m}u_{p+j-1}(S_j)$, minus the padding unary at slot $p+m$, which the longer template
leaves unpadded.  The $j=1$ term of the second sum is $u_p(\square)$, so the difference is the
local margin $r_p-u_p(\square)$ plus the net shift of the remaining suffix unaries minus the
tail padding unary, as stated.  The identity is verified numerically on the released templates
with random unaries and arbitrary closed content.
\end{proof}

\subsection{Commitment schedules}
\label{app:schedule-details}

The causal schedule matrix also rejects a stronger scheduling hypothesis: holding factors,
model, projection step, and support fixed, confidence order beats separator order by 16.7
normalized points and $6.6\times$ in optimum rate on LLaDA, so under exact support, scheduling
should defer uncertainty.

At every step, group-score ties are broken by the lexicographic variable tuple and position-score
ties by position.  An atomic policy takes complete ranked groups until it meets or exceeds the
nominal budget; the split control fills the budget exactly.  Each trace records the budget,
candidate groups, scores, ranking, selected positions, and policy.

For dependence-dispersed batches, let $C_1,\ldots,C_m$ be primal-graph components and let the
declared capacities be $s_b=|C_b|$.  The compiler constructs $B_1,\ldots,B_m$, with
$|B_b|=s_b$, to greedily reduce
\begin{equation}
 R(B)=\sum_b\sum_j \binom{|B_b\cap C_j|}{2}.
 \label{eq:batch-redundancy}
\end{equation}
It processes components by decreasing size and then lexicographically.  Each variable enters the
largest-capacity batch not yet containing its component; if none exists, it enters the largest
remaining batch.  Batch-index ties are deterministic.  This preserves the component-atomic call
budget while separating dependent variables whenever capacity permits.  Proposed batch values are
ranked by mean local model confidence; factor-marginal ranking is an ablation.

For a connected graph, the static separator policy obtains a deterministic greedy min-fill tree
decomposition.  Within each decomposition-tree component, it selects the bag whose removal
minimizes the largest remaining tree component, with ties resolved by the sorted bag tuple.  It
emits unseen variables in numeric order and recurses through remaining tree components
lexicographically.  The residual policy repeats this construction on the still-masked induced
graph.  Separator policies take the first nominal-budget variables and never remask.

The causal matrix crosses model, uniform, random, shuffled-position, and oracle-biased unary potentials with
independently selected schedules.  Oracle-biased unary potentials are non-deployable.  Model, prompt,
factors, exact proposal rule, steps, and feasible support remain fixed in every schedule contrast.

\subsection{Extended benchmark protocols}

The MMLU same-order-copy overlap uses 240 questions from four subjects, independently hashed label
permutations, group sizes four and eight, and 16/32 steps.  The powered causal evaluation uses
1,200 questions in 200 groups from 40 subject-disjoint categories, MAP and sampling, and three
sampling seeds.  The group is the statistical unit.

The all-record NATURAL PLAN task keeps 100 three-person meeting records and adds a \textsc{skip}
value.  Eighty-six records have oracle objective three and 14 have objective two.  Pair factors
encode temporal order and travel; exhaustive support has at most 57 assignments.  The matrix
crosses exact projection, five unary conditions, schedules, MAP/sampling, and a utility-solver
ceiling.

The BFCL-derived task pins all 124 base multi-turn records having a repeated scalar literal.  Each
call chooses the gold call, three type-preserving distractors, or \textsc{unused}; only repeated
literal consistency is hard-constrained.  The official state-and-response checker is isolated in a
fresh namespace for every candidate evaluation.  The corrected development matrix compares
independent, one-shot CP-SAT, iterative factor schedules, global top-five verification, and complete
per-component verification.  It is a controlled derivative, outside the official BFCL leaderboard
submission.

\section{Detailed limitations}
\label{app:detailed-limitations}

\paragraph{Width and scale.}
Exact inference costs $O(d^{w+1})$ per bucket, so densely coupled constraints exceed the declared
budget and are refused before allocation; the 110 large $N$-queens boards and seven-city trip
planning are the measured cases.  Approximate inference such as loopy belief propagation could
cover them, but it gives up the support guarantee that the paper claims, and compilation to
circuits is the exact alternative that we leave to future work.  Device-resident float64 bucket
operations run as standard tensor kernels, and low overhead on width-one and width-two workloads
leaves the cost of large factors open; width-nine Sudoku is the measured stress case.  The
placement audit shows that CUDA is slower than the CPU for tiny 16-entry buckets, so device
placement must follow table size.

\paragraph{Output length and vocabulary.}
Variable-length outputs use a length bound with end-of-text padding, and termination lives in the
automaton state; unbounded generation is outside the evaluated scope.  Open-vocabulary fields are
exact through the class quotient, and the evaluated free-text spans are capped at three tokens.

\paragraph{Baseline scope.}
The finite-support baseline is exact but exhaustive.  It matches the constrained posterior of the
closest finite-automaton method on feasible cases, and it is a semantic reference in place of a
runtime reproduction of the optimized log-depth algorithm of Dang and Ermon.  The CFG, EPIC, and
LAVE Sudoku executions use their official decoders through an adapter that expresses the
clue-derived solution set as a per-instance grammar; this measures a shared finite-language
overlap, and a runtime comparison on code or SMILES remains open.  The executable matrix scores
the direct output before optional completion-based repair, and the CFG and EPIC papers also
report repaired modes, so a direct-output win speaks to the direct setting.  The scalable copy
CFG probe enumerates a fixed-length finite language because copy is outside the context-free
languages.  Our SOAR comparison independently implements the published schedule and defaults
because the official repository lacks a root license, and DINGO's locked official repository
contains no executable source.  These boundaries restrict runtime rankings to the shared Sudoku
workload.

\paragraph{Task scope.}
The MMLU protocol evaluates semantic preservation under a relational constraint.  The significant
LLaDA result is against exact finite support, and the added SOAR and Dream comparisons are
nonsignificant.  GRAM confirms validity without a significant semantic advantage.  NATURAL PLAN
uses records where all meetings are feasible, and candidate-time preprocessing is outside decoding
latency.  The adapters cover the evaluated schemas; database schemas, high-arity scheduling, and
automatic specification parsing remain future work.  Selection experiments render and jointly
score candidates at once, so they leave open whether rendering or joint scoring carries the gain.

\paragraph{Comparisons not run.}
Five comparisons would sharpen the attribution and are left open.  An autoregressive model of
comparable size could use the same factor graph as a per-token feasibility oracle, since a
prefix clamped as evidence leaves a feasible completion exactly when elimination finds nonzero
mass; nothing in the representation argument is specific to arbitrary order, and the comparison
would show what arbitrary-order decoding adds.  On meeting planning, a uniform pool of feasible
candidates scored by the same PLL would separate the model's proposals from the scorer; on
referential JSON the union bound in \S\ref{sec:select} already rules that pool out.  Dream's zero
instructed-record rate on referential JSON, with a zero top-8 pool oracle, was not tested against
Dream's native free-text generation, so an interface cause is bounded only indirectly, by the
Countdown surface-token probe and by Dream copying free text on all 240 open-vocabulary records.
The repaired modes of the CFG and EPIC decoders were not executed.  Selection was not timed end to
end, and variable-length JSON latency at its 7.6M-entry peak tables was not recorded.

\paragraph{Unique completions.}
The Sudoku factors encode public task rules and exclude the gold grid.  When rules and clues admit
a unique completion, exact conditioning can determine it without useful model evidence, so we
report the format+clue control, validity, exact match, and latency separately.  A hard constraint
can also preserve structural validity while selecting a wrong semantic mode.

\section{Complete supplementary result tables}
\label{app:complete-results}

Figure~\ref{fig:systems-evidence}(a) plots each cell's paired repeated-query speedup against
CP-SAT with eight parallel workers, the matrix's single-worker speedup divided by the solver's
own worker speedup in that cell; their geometric mean is $13.64\times$, the value the body quotes.
Against the deterministic single-worker solver of the scaling matrix, the geometric means are
$19.33\times$ for the reused plan and $10.05\times$ when both solvers build their model per query
($7.10\times$ at eight workers).  Repeated projection wins all eight cells under either
configuration, and the one-shot diagnostic, which pays compilation, wins seven.  Each domain--width cell in panel (c) counts five variable sizes.

\begin{table}[ht]
\centering
\small
\begin{tabular}{lccc}
\toprule
& Dream-7B & LLaDA-8B & LLaDA-1.5 \\
Method & \multicolumn{3}{c}{syntax / references / instructed record} \\
\midrule
\emph{Chance floor (analytic)}   & \multicolumn{3}{c}{\emph{0.0009}} \\
\emph{Best fixed template}       & \multicolumn{3}{c}{\emph{0.016}} \\
\midrule
Independent      & 0.000 / 0.000 / 0.000 & 0.000 / 0.000 / 0.000 & 0.000 / 0.000 / 0.000 \\
Automaton only   & 1.000 / 0.474 / 0.000 & 1.000 / 0.480 / 0.186 & 1.000 / 0.516 / 0.224 \\
Relations only   & 0.000 / 1.000 / 0.000 & 0.000 / 1.000 / 0.000 & 0.000 / 1.000 / 0.000 \\
Hybrid MAP       & 1.000 / 1.000 / 0.000 & 1.000 / 1.000 / 0.414 & 1.000 / 1.000 / 0.434 \\
Hybrid top-8 PLL & 1.000 / 1.000 / 0.000 & 1.000 / 1.000 / \textbf{0.920} & 1.000 / 1.000 / \textbf{0.964} \\
\midrule
Top-8 pool oracle & 0.000 & 0.982 & 0.984 \\
\bottomrule
\end{tabular}
\caption{Referential JSON over 500 instances per backbone: syntax validity, referential
validity, and the instructed-record rate.  Neither representation alone enforces both axes; the
hybrid enforces both on every record.  Dream returns valid records that never match the
instruction, and its pool oracle is zero, so its failure is in the constrained distribution's
head, which the scorer can only choose from.}
\label{tab:refjson}
\end{table}

The closest evaluations cover regular languages with DINGO~\citep{suresh2025dingo}, context-free
languages with CFG decoding~\citep{mundler2025cfg} and LAVE~\citep{zhang2026lave}, and global
finite-automaton inference~\citep{dang2026automata}. EPIC supplies the strongest released
CFG-efficiency baseline~\citep{jin2026epic}. Table~\ref{tab:methodology-alignment} maps
their common evidence axes to ours.

\begin{table}[ht]
\centering
\small
\begingroup
\setlength{\tabcolsep}{3pt}
\begin{tabular}{p{0.75in}p{1.1in}p{1.55in}p{1.4in}}
\toprule
Work & Constraint/tasks & Quality evidence & System evidence \\
\midrule
DINGO & Regular; GSM-Symbolic, JSON & accuracy, parse and schema validity & wall time, automaton precomputation \\
CFG & Context-free; C++, JSON, SMILES & syntax, exact match, pass@1 & inference and completion behavior \\
LAVE & Context-free; C++, JSON, SMILES & syntactic@k, functional@k & mean time, lookahead ablations \\
EPIC & Context-free; C++, JSON, SMILES & syntax and functional correctness & normalized time, breakdown, steps, ablations \\
Finite automata & Automata; function calls, Sudoku, planning, SQL, math & accuracy and constraint satisfaction; greedy and sample & latency, throughput, automaton size, step sweep \\
FactorDLM & Factor graphs; Sudoku, coloring, copy, meetings & validity, exact match, task success, agreement, paired inference & latency decomposition, memory, steps, width, CPU/CUDA, failure boundary \\
\bottomrule
\end{tabular}
\endgroup
\caption{Evaluation-methodology alignment with the closest constrained-dLLM systems. Native tasks
differ by representable constraint class. Sudoku supplies the exact shared protocol with finite
automata; the official CFG engine supplies the representation probe.}
\label{tab:methodology-alignment}
\end{table}

\begin{table}[ht]
\centering
\small
\begingroup
\setlength{\tabcolsep}{3pt}
\begin{tabular}{p{0.85in}p{1.35in}p{1.4in}p{1.0in}}
\toprule
Question & Metric and direction & Denominator / paired unit & Evidence \\
\midrule
Direct system quality & complete valid output without repair & 900 Sudoku per model--mode cell & Fig.~\ref{fig:sudoku-sota-quality} \\
Constraint coverage & validity; higher is better & all 900 Sudoku or all 100 plans & Fig.~\ref{fig:quality-evidence} \\
Model utility & normalized objective; invalid is zero & all 100 planning records & Fig.~\ref{fig:quality-evidence}c \\
Schedule quality & semantic-score effect; positive favors dispersed & generated groups within 40 subjects & Fig.~\ref{fig:diagnostic-effects} \\
Workflow selection & execution-success effect & all 124 development records & Fig.~\ref{fig:diagnostic-effects} \\
Exact-system win & CP-SAT/FactorDLM; above one favors FactorDLM & eight optimal cells, 20 queries each & Fig.~\ref{fig:systems-evidence}a \\
Applicability & executed variable sizes inside budget & five sizes per domain--width cell & Fig.~\ref{fig:systems-evidence}c \\
\bottomrule
\end{tabular}
\endgroup
\caption{Metric ledger. Values across rows answer different questions and are not numerically
comparable. Oracle and non-executed cross-paper values are labeled in their figure captions.}
\label{tab:metric-ledger}
\end{table}

\begin{figure}[ht]
  \centering
  \includegraphics[width=\linewidth]{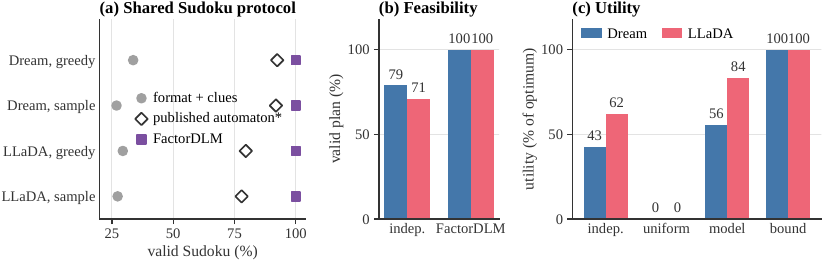}
  \caption{Task-quality evidence with metric and denominator separated by panel.  Panels (a) and
  (b) stand: shared-protocol Sudoku validity against the non-executed published automaton result,
  and planning feasibility.  \textbf{Panel (c)'s \emph{uniform} bar is a degenerate control}: it is
  zero because deterministic ties decline every meeting, so its margin reflects the tie-break.
  In panel (c), \emph{uniform} and \emph{model} are FactorDLM under uniform and model unaries.  The
  informative comparison for that panel is against the non-degenerate model-free controls in
  Table~\ref{tab:natural-plan}, where model utility is established on LLaDA and refuted on Dream.}
  \label{fig:quality-evidence}
\end{figure}

\begin{figure}[ht]
  \centering
  \includegraphics[width=\linewidth]{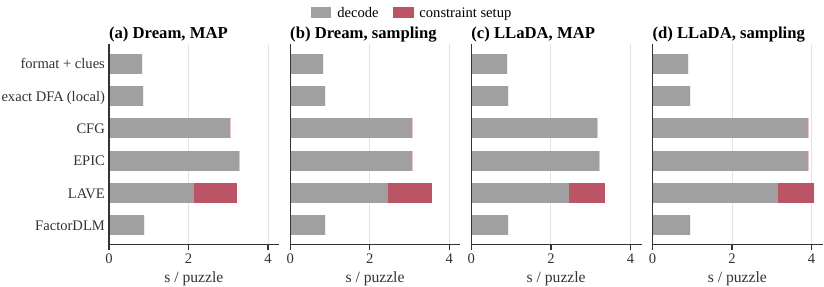}
  \caption{Matched primary-cell cost.  Each bar separates synchronized decode from per-puzzle
  grammar/checker setup; static local plans are compiled once and reused.  Lower is better.}
  \label{fig:sudoku-sota-latency}
\end{figure}

\begin{figure}[ht]
  \centering
  \includegraphics[width=\linewidth]{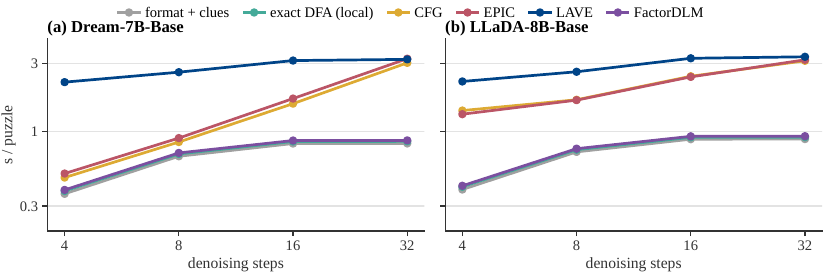}
  \caption{MAP setup-plus-decode scaling on the same clue-stratified 90-puzzle subset at
  4, 8, 16, and 32 denoising steps.}
  \label{fig:sudoku-sota-scaling}
\end{figure}

\begin{figure}[ht]
  \centering
  \includegraphics[width=\linewidth]{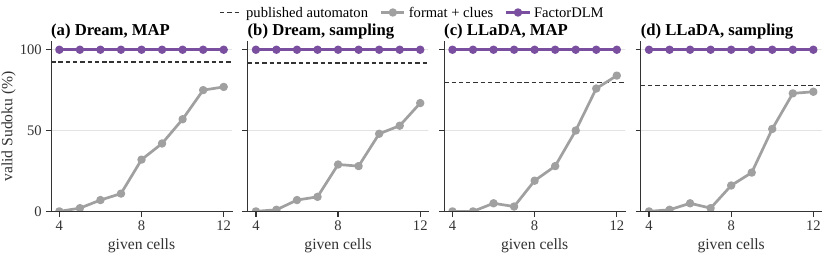}
  \caption{Official Sudoku validity by clue count. FactorDLM enforces rows, columns, and boxes; the
  local control and published finite automaton enforce only format and given clues.}
  \label{fig:sudoku}
\end{figure}

\begin{figure}[ht]
  \centering
  \includegraphics[width=\linewidth]{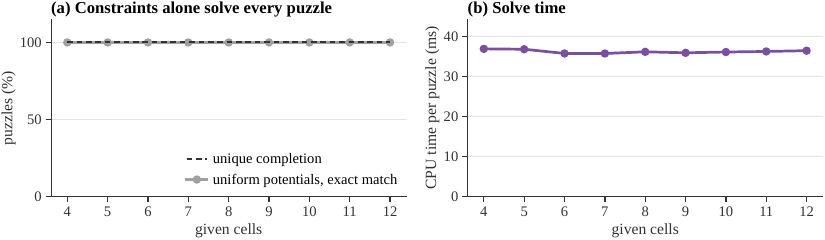}
  \caption{Solver-only Sudoku audit. Uniform unary potentials recover the unique official completion on all
  900 puzzles, proving that the shared-task result measures relational coverage rather than improved
  model selection.}
  \label{fig:sudoku-solver-audit}
\end{figure}

Table~\ref{tab:headroom} records the exploratory oracle-gap relationship referenced in
\S\ref{sec:select}: the realized selection gain tracks how much room greedy decoding leaves,
across eight cells spanning two tasks, two backbones, and four problem sizes.

\begin{table}[ht]
\centering
\small
\begin{tabular}{lrrrr}
\toprule
Cell & Greedy & Best-of-$K$ oracle & Oracle gap & Realised gain \\
\midrule
Dream trip & 0.921 & 0.963 & 0.042 & $-0.093$ \\
LLaDA trip & 0.847 & 0.954 & 0.106 & $-0.125$ \\
\midrule
Dream meeting 5p & 0.051 & 0.276 & 0.224 & $+0.122$ \\
LLaDA meeting 5p & 0.133 & 0.367 & 0.235 & $+0.082$ \\
LLaDA meeting 4p & 0.410 & 0.750 & 0.340 & $+0.160$ \\
Dream meeting 4p & 0.040 & 0.400 & 0.360 & $+0.250$ \\
Dream meeting 3p & 0.120 & 0.500 & 0.380 & $+0.370$ \\
LLaDA meeting 3p & 0.530 & 0.970 & 0.440 & $+0.440$ \\
\bottomrule
\end{tabular}
\caption{Exploratory relation between oracle gap and reranking gain across two task families, two
backbones and four problem sizes, with Pearson $r=0.955$.  Trip-planning
cells are golden-itinerary match; meeting cells are exact-optimum rate.  The prospective calendar
holdout rejects the fitted law on LLaDA, so this figure is descriptive and does not validate
a general predictor or adaptive deployment rule.}
\label{tab:headroom}
\end{table}

\subsection{Evaluation audits}

Each audit below asks whether a published metric is reachable without model input.  The taxonomy
table names the mechanism per benchmark, the target tables give the model-free rates that match
or beat every dLLM method, and the control table separates the degenerate uniform-MAP baseline
from a proper random control.  These are the analyses summarized in \S\ref{sec:audits}.

\begin{table}[ht]
\centering
\small
\begin{tabular}{llrrr}
\toprule
People & Backbone & model $-$ uniform (sampling) & model $-$ shuffled (MAP) & Records \\
\midrule
4 & Dream & $+0.003$ [$-0.031$, $+0.038$] & $+0.023$ [$-0.006$, $+0.052$] & 100 \\
4 & LLaDA & $+0.063$ [$+0.020$, $+0.105$] & $-0.018$ [$-0.049$, $+0.014$] & 100 \\
5 & Dream & $-0.058$ [$-0.091$, $-0.025$] & $+0.015$ [$-0.025$, $+0.055$] & 98 \\
5 & LLaDA & $-0.074$ [$-0.118$, $-0.029$] & $+0.018$ [$-0.025$, $+0.061$] & 98 \\
\bottomrule
\end{tabular}
\caption{Extended NATURAL PLAN, paired record-level differences in normalized objective with 95\%
bootstrap intervals over 20,000 resamples.  The model-utility criterion passes 5 of 8 cells
and the position-semantics criterion 0 of 8; every passing MAP cell is measured against the degenerate
all-\textsc{skip} control.  At five people the model is significantly worse than uniform unary potentials on
both backbones.}
\label{tab:natural-plan-extended}
\end{table}

\begin{table}[ht]
\centering
\small
\begin{tabular}{llrrrrr}
\toprule
Backbone & People & Agreement & Judge, sorted & Judge, permuted & Judge, symmetrized & PLL \\
\midrule
Dream & 3 & 0.627 & 0.320 & 0.267 & 0.310 & \textbf{0.490} \\
Dream & 4 & 0.487 & 0.270 & 0.235 & \textbf{0.290} & \textbf{0.290} \\
Dream & 5 & 0.430 & 0.143 & 0.146 & \textbf{0.184} & 0.173 \\
LLaDA & 3 & 0.729 & \textbf{0.970} & 0.850 & \textbf{0.970} & \textbf{0.970} \\
LLaDA & 4 & 0.642 & \textbf{0.730} & 0.643 & 0.710 & 0.570 \\
LLaDA & 5 & 0.543 & 0.296 & 0.276 & \textbf{0.357} & 0.214 \\
\midrule
LLaDA-1.5 & 3 & 0.780 & 0.950 & 0.908 & \textbf{0.980} & 0.970 \\
\bottomrule
\end{tabular}
\caption{Judge order probe over the matched candidate pools: exact-optimum rate under the
recorded sorted candidate order, averaged over ten seeded permutations, and under
order-symmetrized scoring over all eleven orderings, with the order-free pseudo-likelihood
scorer for reference.  Agreement is the fraction of the ten permuted trials per record that
repeat the sorted-order choice.  Every cell is order-sensitive and none reaches the 0.80
stability threshold fixed in advance; the LLaDA-8B advantage at four and five people survives permutation
and is a content result, while its three-person tie is ordering-assisted.}
\label{tab:judge-order}
\end{table}

\begin{figure}[ht]
  \centering
  \includegraphics[width=\linewidth]{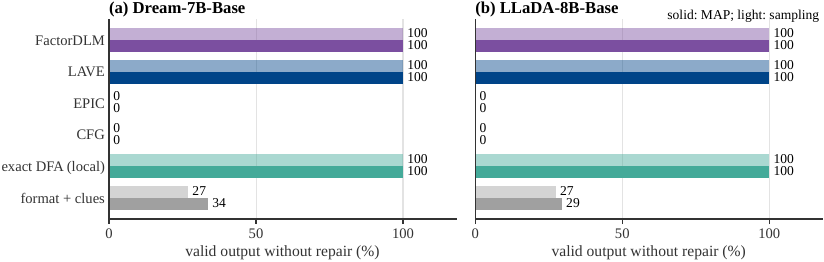}
  \caption{Direct fixed-budget validity on the matched 900-puzzle protocol.  LAVE, the local exact
  DFA, and FactorDLM tie at 100\% in every cell.  CFG and EPIC are shown with
  completion-based repair disabled (their Con.$^{-}$ and E.$^{-}$ configurations); within the
  fixed 16-slot budget neither reaches an accepting string on any puzzle, so their zeros measure
  the direct setting and say nothing about their repaired modes, which the figure does not rank.}
  \label{fig:sudoku-sota-quality}
\end{figure}

\begin{figure}[ht]
\centering
\includegraphics[width=\linewidth]{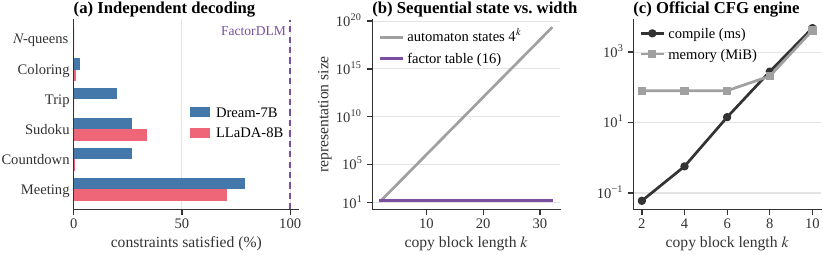}
\caption{\textbf{Relational constraints defeat independent decoding, and sequential encodings of
them are expensive.}  (a) Constraint satisfaction of independent decoding over the identical prompt
and slot layout: near zero where the relation graph is dense, and never complete.  FactorDLM is
exact by construction.  (b) The same-order copy relation $x_i=x_{k+i}$ has induced width one and a
constant 16-entry table at every $k$, while a deterministic automaton must distinguish all $4^k$
first blocks.  (c) That growth is not merely asymptotic: on the official CFG engine, the finite
restriction of the same relation costs 0.06~ms and 80~MiB at $k=2$ and 4.8~s and 4.1~GiB at
$k=10$.}
\label{fig:motivation}
\end{figure}

\begin{table}[ht]
\centering
\small
\begin{tabular}{lrrrrrr}
\toprule
& \multicolumn{2}{c}{Dream-7B} & \multicolumn{2}{c}{LLaDA-8B} & \multicolumn{2}{c}{LLaDA-1.5} \\
\cmidrule(lr){2-3} \cmidrule(lr){4-5} \cmidrule(lr){6-7}
Budget & Rerank & Oracle & Rerank & Oracle & Rerank & Oracle \\
\midrule
greedy  & \multicolumn{2}{c}{0.120} & \multicolumn{2}{c}{0.530} & \multicolumn{2}{c}{0.650} \\
$K=2$   & 0.240 & 0.250 & 0.790 & 0.790 & 0.840 & 0.840 \\
$K=4$   & 0.360 & 0.370 & 0.930 & 0.940 & 0.930 & 0.940 \\
$K=8$   & 0.490 & 0.500 & 0.970 & 0.970 & 0.970 & 0.980 \\
\bottomrule
\end{tabular}
\caption{Candidate-budget sweep behind Figure~\ref{fig:frontier}: exact-optimum rate of
pseudo-likelihood reranking and of the best-of-$K$ oracle on three-person meeting planning.
The reranker sits on or within one record of the oracle at every budget on every backbone, and
every rerank point beats greedy with exact McNemar $p \le 3.9\times10^{-6}$.}
\label{tab:ksweep}
\end{table}

\begin{table}[ht]
\centering
\small
\begingroup
\setlength{\tabcolsep}{3.5pt}
\caption{The shortcut analyses behind Figure~\ref{fig:shortcuts}: what determines each published
metric, and the model-free rule that achieves it.}
\label{tab:audits}
\begin{tabular}{p{1.05in}p{1.05in}p{1.35in}p{1.6in}}
\toprule
Benchmark & Metric & What determines it & Shortcut result \\
\midrule
Official Sudoku (900) & exact match & unique feasible completion & model-free uniform unary potentials reach 900/900 in 36.22~ms \\
Same-order copy & validity & constraint, by construction & second block is a function of the first \\
GRAM N-Queens & target agreement & canonical labeling & target is the lexicographic minimum in 405/405 multi-solution instances \\
GRAM coloring & target partition & canonical labeling & vertex-order greedy reproduces the target on 309/309 graphs; every dLLM method scores below chance \\
NATURAL PLAN & normalized objective & a degenerate control & uniform-MAP scores zero by tie-break; against random unary potentials LLaDA gains 21 points and Dream loses \\
\bottomrule
\end{tabular}
\endgroup
\end{table}

\subsection{Diagnostic schedule and workflow studies}

The original 240-question same-order-copy overlap gave a significant LLaDA component-schedule gain
over exhaustive support, but the 1,200-question, 40-subject-disjoint replication does
not support a semantic-quality improvement.  Relative to individual model-confidence commitment,
dispersed batches change MAP semantic score by $-1.08$ points on Dream (95\% cluster interval
$[-4.00,1.83]$) and $-0.25$ on LLaDA ($[-3.08,2.58]$).  Grouped commitment nevertheless cuts
latency by 952~ms on Dream and 1,149~ms on LLaDA, with both paired intervals excluding zero.  It
reduces redundant sequential calls but is not a quality mechanism.

The BFCL-derived workflow evaluation exposed mutable official-checker state across candidate calls.
The corrected adapter isolates and cleans every checker namespace; both 1,116-trial runs have
zero exact-plan invariant failures.  Complete per-component enumeration raises gold-candidate
coverage from 38.7\%/39.5\% to 100\%, but the factorized verifier obtains only 16.13\% Dream and
15.32\% LLaDA execution success, versus 15.32\%/12.90\% for one-shot CP-SAT and
16.13\%/16.13\% for component decoding.  It costs 3.33/4.14 seconds per record and passes 4 of
its 16 acceptance gates, so the disjoint validation is withheld.  These diagnostics localize the
unsupported effect to schedule quality and model-based selection inside feasible support.

\begin{table}[ht]
\centering
\small
\begin{tabular}{llrrrr}
\toprule
Model & Method & Valid (\%) & Partition (\%) & Agreement (\%) & ms/output \\
\midrule
Dream & Independent & 2.68 & 0.00 & 55.47 & 69.03 \\
      & Finite support & 100.00 & 13.87 & 76.02 & 71.71 \\
      & \textbf{Factor separator} & \textbf{100.00} & \textbf{16.78} & \textbf{76.92} & \textbf{71.16} \\
\midrule
LLaDA & Independent & 0.89 & 0.00 & 54.88 & 76.57 \\
      & Finite support & 100.00 & 15.88 & 77.52 & 79.18 \\
      & \textbf{Factor separator} & \textbf{100.00} & \textbf{16.78} & \textbf{77.82} & \textbf{78.66} \\
\bottomrule
\end{tabular}
\caption{GRAM graph coloring on the complete public test split.  Validity is the informative
column: independent decoding almost always violates an edge, and exact projection satisfies every
edge.  \textbf{The partition column measures agreement with a canonical target, so it gives little
evidence of quality.}  We audit it in \S\ref{sec:audits} and find that the published target is the
lexicographic-minimum coloring, which vertex-order greedy reproduces on 309 of 309 graphs where
three colors suffice.  Every method here
therefore scores below the 0.191 rate of choosing uniformly among partitions, and the column
measures agreement with a two-line deterministic routine.}
\label{tab:gram}
\end{table}

\begin{table}[ht]
\centering
\small
\begin{tabular}{llrrrr}
\toprule
Model & Method (MAP unless noted) & Valid (\%) & Objective (\%) & \textsc{skip} rate & Optimum (\%) \\
\midrule
Dream & Independent model & 79.00 & 42.83 & --- & 6.00 \\
      & Uniform factor & 100.00 & 0.00 & 1.000 & 0.00 \\
      & \emph{Uniform factor, sampling} & 100.00 & 57.94 & 0.450 & 22.30 \\
      & \emph{Random factor} & 100.00 & 62.61 & 0.401 & 24.30 \\
      & \emph{Shuffled-position factor} & 100.00 & 52.78 & 0.497 & 12.30 \\
      & Factor model & 100.00 & 55.83 & 0.467 & 12.00 \\
      & Utility solver & 100.00 & 100.00 & --- & 100.00 \\
\midrule
LLaDA & Independent model & 71.00 & 62.00 & --- & 30.00 \\
      & Uniform factor & 100.00 & 0.00 & 1.000 & 0.00 \\
      & \emph{Uniform factor, sampling} & 100.00 & 57.94 & 0.450 & 22.30 \\
      & \emph{Random factor} & 100.00 & 62.61 & 0.401 & 24.30 \\
      & \emph{Shuffled-position factor} & 100.00 & 86.56 & 0.173 & 65.00 \\
      & Factor model & 100.00 & 83.50 & 0.203 & 53.00 \\
      & Utility solver & 100.00 & 100.00 & --- & 100.00 \\
\bottomrule
\end{tabular}
\caption{All 100 three-person NATURAL PLAN records, with the non-degenerate controls in italics.
Objective is the scheduled-meeting count normalized by the instance optimum.  Uniform MAP scores
zero only because deterministic ties select \textsc{skip} everywhere, so its margin is not evidence.
Against the sampling and random controls, Dream's model unary potentials are \emph{worse} than model-free
ones, and on both backbones the objective tracks the \textsc{skip} rate rather than the unary
source.  Uniform and random are model-free and therefore identical across backbones, which is a
correctness check on the controls.}
\label{tab:natural-plan}
\end{table}

\begin{table}[ht]
\centering
\small
\begin{tabular}{llrrr}
\toprule
Model & BFCL-derived method & Execution (\%) & Exact plan (\%) & ms/output \\
\midrule
Dream & Format-only (typed slots) & 1.61 & 1.61 & 118.20 \\
      & One-shot CP-SAT & 15.32 & 15.32 & 136.96 \\
      & Component factor & 16.13 & 15.32 & 593.44 \\
      & Global verifier & 10.48 & 9.68 & 3,443.28 \\
      & Factorized verifier & 16.13 & 16.13 & 3,333.92 \\
\midrule
LLaDA & Format-only (typed slots) & 0.81 & 0.81 & 130.55 \\
      & One-shot CP-SAT & 12.90 & 12.90 & 170.34 \\
      & Component factor & 16.13 & 15.32 & 742.79 \\
      & Global verifier & 16.94 & 16.13 & 4,248.35 \\
      & Factorized verifier & 15.32 & 15.32 & 4,141.66 \\
\bottomrule
\end{tabular}
\caption{The 124-record executable-workflow study, official checker.  The format-only arm decodes
every slot from its own typed candidate list, so all of its outputs are individually well-typed,
yet only 14.5\% and 12.9\% satisfy the cross-call relations and execution collapses to 1.6\% and
0.8\%, so the residual failures on this task are relational.  Relational factors lift execution to
16.1\% on both backbones (18W/0L, $p=7.6\times10^{-6}$; 19W/0L, $p=3.8\times10^{-6}$) and match
a solver handed the same candidates (differences not significant).  Factorized verification has
complete candidate coverage but passed only 4 of 16 development criteria.}
\label{tab:bfcl-corrected}
\end{table}

\begin{figure}[ht]
  \centering
  \includegraphics[width=\linewidth]{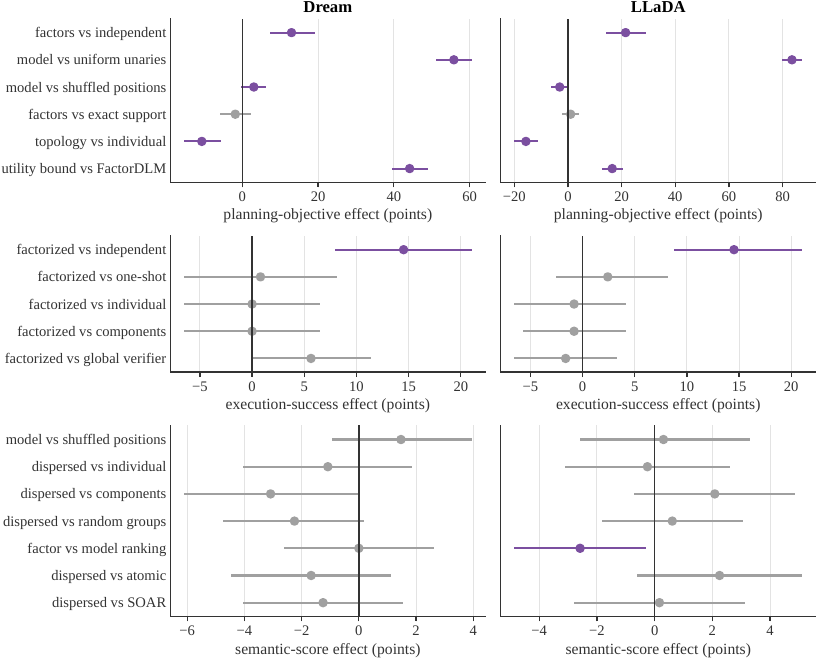}
  \caption{Complete MAP paired effects; points show estimates and bars show 95\%
  paired-unit bootstrap intervals, in purple when the interval excludes zero. Positive favors the
  first named method. Rows separate planning
  objective over 100 records, workflow execution over 124 records, and same-order semantic score
  over 200 generated groups (1,200 questions). Constraint and model-utility effects are positive,
  while factorized workflow selection and dispersed schedule quality do not beat strong controls.}
  \label{fig:diagnostic-effects}
\end{figure}

\begin{table}[ht]
\centering
\small
\begin{tabular}{rrrll}
\toprule
$k$ & Finite-state lower bound & Factor peak & Finite support & FactorDLM \\
\midrule
4  & 256 & 16 & runs & 100\% valid \\
8  & 65,536 & 16 & runs & 100\% valid \\
16 & 4,294,967,296 & 16 & state cap & 100\% valid \\
32 & 18,446,744,073,709,551,616 & 16 & state cap & 100\% valid \\
\bottomrule
\end{tabular}
\caption{Native representation and GPU decoding scale for four-way same-order copy.}
\label{tab:scaling}
\end{table}

\begin{table}[ht]
\centering
\small
\begin{tabular}{rrrrrr}
\toprule
$k$ & Alternatives & Grammar MB & Compile ms & Query ms & RSS MiB \\
\midrule
2  & 16        & 0.0002 & 0.06    & 0.04    & 80.0 \\
4  & 256       & 0.0046 & 0.57    & 0.43    & 80.0 \\
6  & 4,096     & 0.1065 & 14.18   & 11.70   & 80.0 \\
8  & 65,536    & 2.2282 & 278.48  & 377.01  & 206.9 \\
10 & 1,048,576 & 44.0402& 4,764.55& 8,238.87& 4,114.6 \\
\bottomrule
\end{tabular}
\caption{Official CFG engine on finite same-order-copy grammars.}
\label{tab:cfg}
\end{table}

\begin{figure}[ht]
  \centering
  \includegraphics[width=\linewidth]{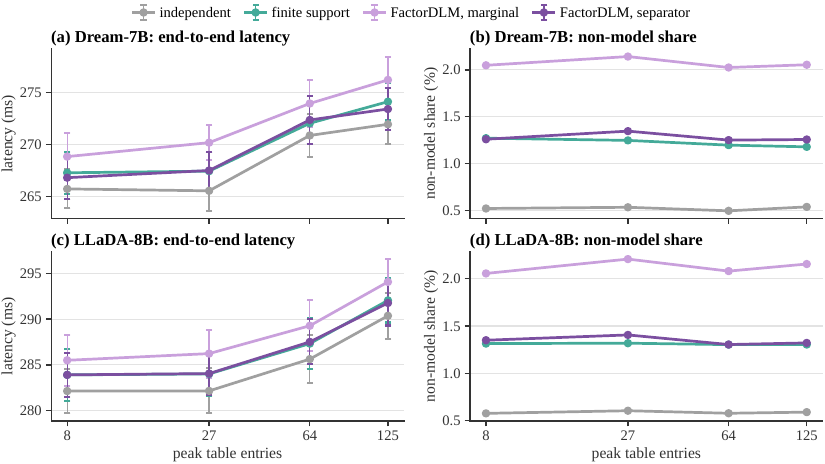}
  \caption{Controlled A6000 runtime. Error bars are standard deviations over 20 repeats per
  method--record cell; the right column decomposes synchronized end-to-end time.}
  \label{fig:runtime}
\end{figure}

\begin{figure}[ht]
  \centering
  \includegraphics[width=\linewidth]{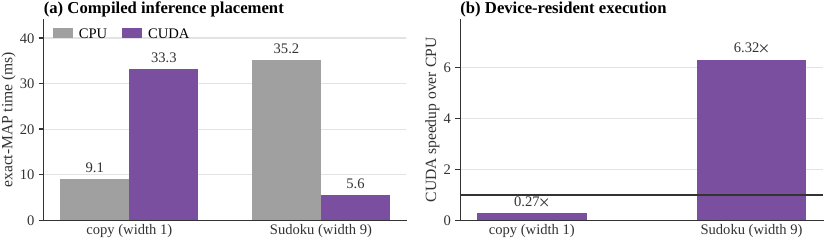}
  \caption{Exact CPU/CUDA placement validation over 100 balanced repetitions per workload.  CUDA
  is $6.32\times$ faster on the width-nine Sudoku graph, the CPU is $3.65\times$ faster when
  width-one tables are too small to amortize kernel launches, and all 400 cross-device
  comparisons agree exactly.}
  \label{fig:device-validation}
\end{figure}

\begin{figure}[ht]
  \centering
  \includegraphics[width=\linewidth]{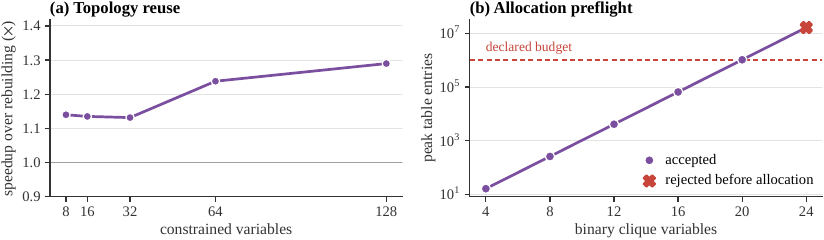}
  \caption{Compiled-plan CPU speedup and fail-fast allocation boundary. The 24-variable clique is
  rejected symbolically before its dense table allocates.}
\label{fig:production-validation}
\end{figure}

\begin{figure}[ht]
  \centering
  \includegraphics[width=\linewidth]{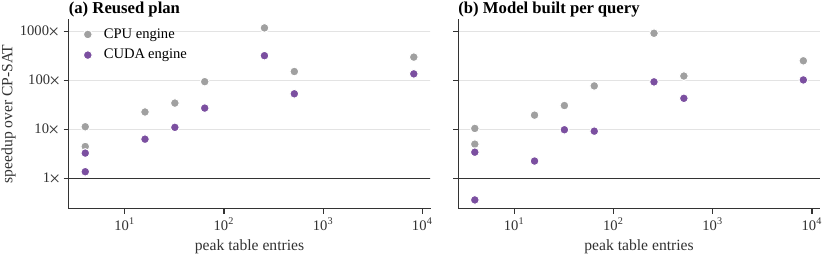}
  \caption{Pinned CP-SAT divided by FactorDLM latency on the eight cells where every solver query
  is optimal.  Repeated projection reuses the factor plan; one-shot includes compilation.}
  \label{fig:solver-scaling}
\end{figure}

\begin{table}[ht]
\centering
\small
\begin{tabular}{lrrrr}
\toprule
Method & Cells & vs.\ 1-worker CP-SAT & 95\% CI & vs.\ 8-worker CP-SAT \\
\midrule
Factor CPU & 8 & 61.06$\times$ & [19.43, 201.84] & \textbf{43.11$\times$} \\
Factor CUDA & 8 & 19.33$\times$ & [5.82, 64.73] & \textbf{13.64$\times$} \\
\bottomrule
\end{tabular}
\caption{Repeated changing-unary projection, geometric means over equally weighted fully optimal
solver cells.  All 320 listed factor--solver objective comparisons agree exactly.  The scaling
matrix pins one search worker for determinism; the audit re-times the same queries with
\texttt{num\_search\_workers = 8}, OR-Tools' multi-worker portfolio, which is $1.42\times$ faster
here, so the bolded column is the stricter comparison and the one quoted in the main text.  FactorDLM wins all eight cells under both configurations.}
\label{tab:solver-scaling}
\end{table}

\end{document}